%% file: LearnPoly.tex
\documentclass[12pt]{article}
\usepackage[numbers, compress]{natbib}
\usepackage[utf8]{inputenc} % allow utf-8 input
\usepackage[T1]{fontenc}    % use 8-bit T1 fonts
\usepackage{hyperref}       % hyperlinks
\usepackage{url}            % simple URL typesetting
\usepackage{booktabs}       % professional-quality tables
\usepackage{amsfonts}       % blackboard math symbols
\usepackage{nicefrac}       % compact symbols for 1/2, etc.
\usepackage{microtype}      % microtypography
\usepackage{xcolor}         % colors

\usepackage{amsmath, nccmath, amssymb}
\usepackage{amsthm}
\usepackage[ruled]{algorithm2e}
\usepackage{cleveref}
\usepackage{graphicx}
\usepackage{tcolorbox}
\usepackage{mathrsfs}
\usepackage{graphbox}
\usepackage{caption}
\usepackage{subcaption}
\usepackage{wrapfig}
\usepackage{float}
\usepackage{enumitem}
\usepackage{tabu}
\usepackage{doi}

\usepackage{times}
\usepackage[margin=1.25in]{geometry}

\newtheorem{theorem}{Theorem}
\newtheorem{lemma}{Lemma}
\newtheorem{assumption}{Assumption}
\newtheorem{definition}{Definition}

\newtheorem{corollary}{Corollary}

\title{Neural Networks can Learn Representations with Gradient Descent}

\graphicspath{ {images/} }
\include{macros}

\include{notation}

\newcommand{\bw}{\overline{w}}

\begin{document}

% paragraph formatting
\setlength{\parindent}{0pt}
\setlength{\parskip  }{5.5pt}

\title{Neural Networks can Learn Representations with Gradient Descent}

\author{%
  Alex Damian \\
  Princeton University\\
  \texttt{ad27@princeton.edu}
  \and
  Jason D. Lee \\
  Princeton University \\
  \texttt{jasonlee@princeton.edu}
  \and
  Mahdi Soltanolkotabi \\
  University of Southern California \\
  \texttt{soltanol@usc.edu}
}

\maketitle

\begin{abstract}%
Significant theoretical work has established that in specific regimes, neural networks trained by gradient descent behave like kernel methods. However, in practice, it is known that neural networks strongly outperform their associated kernels. In this work, we explain this gap by demonstrating that there is a large class of functions which cannot be efficiently learned by kernel methods but can be easily learned with gradient descent on a two layer neural network outside the kernel regime by learning representations that are relevant to the target task. We also demonstrate that these representations allow for efficient transfer learning, which is impossible in the kernel regime.

Specifically, we consider the problem of learning polynomials which depend on only a few relevant directions, i.e. of the form $f^\star(x) = g(Ux)$ where $U: \R^d \to \R^r$ with $d \gg r$. When the degree of $f^\star$ is $p$, it is known that $n \asymp d^p$ samples are necessary to learn $f^\star$ in the kernel regime. Our primary result is that gradient descent learns a representation of the data which depends only on the directions relevant to $f^\star$. This results in an improved sample complexity of $n\asymp d^2 r + dr^p$. Furthermore, in a transfer learning setup where the data distributions in the source and target domain share the same representation $U$ but have different polynomial heads we show that a popular heuristic for transfer learning has a target sample complexity independent of $d$.
\end{abstract}

\section{Introduction}
\input{introduction}

\section{Setup}
\input{setup}

\section{Main Results}
\input{results}

\section{Related work}\label{section:relatedwork}
\input{related_work}
\section{Proof Sketches}
\input{sketch}

\section{Experiments}\label{section:experiments}
\input{experiments}

\section{Discussion and Future Work}
\input{discussion}

\section{Acknowledgements}
AD acknowledges support from a NSF Graduate Research Fellowship. JDL and AD acknowledge support of the ARO under MURI Award W911NF-11-1-0304, the Sloan Research Fellowship, NSF CCF 2002272, NSF IIS 2107304, ONR Young Investigator Award, and NSF-CAREER under award \#2144994. MS is supported by the Packard Fellowship in Science and Engineering, a Sloan Fellowship in Mathematics, an NSF-CAREER under award \#1846369, DARPA Learning with Less Labels (LwLL) and FastNICS programs, and NSF-CIF awards \#1813877 and \#2008443.

\newpage
\bibliography{simonduref,myref,ref,Bibfiles,Bibfiles2,literature,beyondntk}
\newpage
\appendix
\input{appendix}

\end{document}

%% file: macros.tex
\usepackage{mathrsfs}
\usepackage{physics}
\usepackage{tensor}

\let\P\relax\DeclareMathOperator{\P}{\mathbb{P}}
\DeclareMathOperator{\E}{\mathbb{E}}
\DeclareMathOperator{\Var}{Var}

\newcommand{\R}{\mathbb{R}}

\newcommand{\1}{\mathbf{1}}

\newcommand{\iprod}[1]{\left\langle #1 \right\rangle}

\DeclareMathOperator{\spn}{span}
\DeclareMathOperator{\unif}{Unif}
\DeclareMathOperator{\poly}{poly}
\DeclareMathOperator{\polylog}{polylog}
\DeclareMathOperator{\relu}{ReLU}
\DeclareMathOperator{\mat}{Mat}

\let\vec\relax\DeclareMathOperator{\vec}{Vec}
\DeclareMathOperator{\sym}{Sym}

\ifdefined\usebigfont

\usepackage{times}
\usepackage[fontsize=13pt]{scrextend}
\AtBeginDocument{
	\newgeometry{letterpaper,left=1.56in,right=1.56in,top=1.71in,bottom=1.77in}
}
\else
\fi

%% file: notation.tex
\usepackage{bm}

%% file: introduction.tex
Crucial to the practical success of deep learning is the ability of gradient-based algorithms to learn good feature representations from the training data and learn simple functions on top of these representations. Despite significant progress towards a theoretical foundation for neural networks, a robust understanding of this unique representation learning capability of gradient descent methods has remained elusive. A major challenge is that due to the highly nonconvex loss landscape, establishing convergence to a global optimum that achieves near zero training loss is challenging. Furthermore, due to the overparameterized nature of modern neural nets (containing many more parameters than training data) the training landscape has many global optima. In fact, there are many global optima with poor generalization performance~\citep{zhang2016understanding,liu2020bad}. This paper thus focuses on answering this intriguing question:

\begin{quote}
\center
How do gradient-based methods learn feature representations and why do these representations allow for efficient generalization and transfer learning?
\end{quote}
\vspace{1em}

 The most prominent contemporary approach to understanding neural networks is the linearization or neural tangent kernel (NTK)~\citep{soltanolkotabi2018theoretical,jacot2018neural} technique. The premise of the linearization method is that the dynamics of gradient descent are well-approximated by gradient descent on a linear regression instance with fixed feature representation. Using this linearization technique, it is possible to prove convergence to a zero training loss point~\citep{soltanolkotabi2018theoretical,du2018gradienta, du2018gradientb}. However, this technique often requires unrealistic hyper-parameter choices (e.g.~small learning rate, large initialization, or wide networks) that does not allow the features to evolve across the iterations and thus the generalization error with this technique cannot be better than that of a kernel method. Indeed, precise lower bounds show that the NTK solutions do not generalize better than the polynomial kernel~\citep{ghorbani2019limitations}. As a result this regime of training is also sometimes referred to as the \emph{lazy} regime \cite{chizat2019lazy}.\footnote{See Section \ref{section:relatedwork} for a more in depth discussion of this literature and other related work.} In practice, neural networks far outperform their corresponding induced kernels~\citep{arora2019exact}. Therefore, understanding the representation learning of neural networks beyond the lazy regime is of fundamental importance.
 
 In this paper, we initiate the study of the representation learning of neural networks beyond this NTK/linear/lazy regime. To this aim, we consider the problem of learning polynomials with low-dimensional latent representation of the form $f^*(x) = g(Ux)$, where $U$ maps from $d$ to  $ r $ dimensions with $d\gg r$ with $g$ a multivariate polynomial of degree $p$. This is a natural choice as the failure of the NTK solution is in part due to its inability to learn data-dependent feature representations that adapt to the intrinsic low latent dimensionality of the ground truth function. Existing analysis based on the NTK regime provably require $n \asymp d^p$ samples~\citep{ghorbani2019limitations} to learn \textit{any} degree $p$ polynomial, even if they only depend on a few relevant directions. In contrast we show that gradient descent from random initialization only requires $n \asymp d^2r + dr^p$ samples, breaking the sample complexity barrier dictated by NTK proof techniques. More specifically, our contributions are as follows:

 %In this paper, we study ground truth functions of the form $f^*(x) = g(Ux)$, that is functions that only depend on an $r$-dimensional subspace of the $d$-dimensional input.

 %When $g$ is a degree $g$ polynomial, existing analysis using equivalence to kernel methods require $n \asymp d^p$ samples.  Our primary result is that gradient descent can learn the representation $\text{span} (U)$, and then perform a kernel method restricted to the span. This results  in a sample complexity of $n\asymp XXX$. Finally, we show that the assumptions in our algorithm is necessary by showing a correlational statistical query lower bound of $n \asymp d^{O(p)}$ when the assumptions are violated.
 \begin{enumerate}
 	\item \textbf{Feature Learning:} When the target function $f^\star = g(Ux)$ only depends on the projection of $x$ onto a hidden subspace $\spn(U)$, we show that gradient descent learns features that span $\spn(U)$. Leveraging these features, gradient descent can reach vanishing training loss with a very small network which guarantees good generalization performance. See \Cref{sec:sketch:samplecomplexity}.
 	\item \textbf{Improved Sample Complexity:} Using classical generalization theory, we demonstrate that when $f^\star: \R^d \to \R$ is a polynomial of degree $p$ which depends on $r$ relevant dimensions (\Cref{assumption:rdim}), gradient descent on a two layer neural network learns $f^\star$ with only $n \asymp d^2 r + dr^p$ samples. This contrasts with the lower bound for random features/NTK methods which require $d^p$ samples to learn \textit{any} degree $p$ polynomial. See \Cref{thm:sample_complexity}.
 	
 	\item \textbf{Transfer learning:} We show that when the target task ground truth is $f_{\text{target}}^\star (x) = \tilde g(Ux)$, then by simply retraining the network head, gradient descent learns $f_{\text{target}}^\star$ with only $N \asymp r^{p}$ target samples and width $m \asymp r^p$, which is independent of the ambient dimension $d$. In contrast, learning from scratch would require $N \asymp d^{\Omega(p)}$ target samples.
 	\item \textbf{Lower Bound:} Finally, we show a lower bound that demonstrates our non-degeneracy assumption (Assumption \ref{assumption:C2rank}) is strictly necessary. Without the non-degeneracy, there is a family of polynomials which depend on single relevant dimensions (i.e. of the form $f^\star(x) = g(u \cdot x)$) which cannot be learned with fewer than $n \asymp d^{p/2}$ by any gradient descent based learner.
 \end{enumerate}

%% file: setup.tex
\label{secsetup}
\subsection{Input Distribution and Target Function}
In this paper we focus on learning a target function $f^\star(x): \R^d \to \R$ over the input distribution $\mathcal{D} := N(0,I_d)$. We assume that $f^\star$ is a degree $p$ polynomial, normalized so that $\E_{x \sim \mathcal{D}}[f^\star(x)^2] = 1.$ We will attempt to learn $f^\star$ given $n$ i.i.d. datapoints $\qty{x_i,y_i}_{i \in [n]}$ with 
\begin{align*}
	x_i \sim \mathcal{D} \qc y_i = f^\star(x_i) + \epsilon_i \qand \epsilon_i \sim \{-\varsigma,\varsigma\}
\end{align*}
where $\varsigma^2$ controls the strength of the label noise.

In order to make the problem of learning $f^\star$ tractable, additional assumptions are necessary. The set of degree $p$ polynomials in $d$ dimensions span a linear subspace of $L^2(\mathcal{D})$ of dimension $\Theta(d^p)$. Learning arbitrary degree $p$ polynomials therefore requires $n \gtrsim d^p$ samples. We follow \citet{chen2020learning,chen2020towards} in assuming that the ground truth $f^\star$ has a special low dimensional latent structure. Specifically, we assume that $f^\star$ only depends on a small number of relevant dimensions and that the expected Hessian is non degenerate. We show in \Cref{thm:CSQ_lower_bound} that this non degeneracy assumption is strictly necessary to avoid sample complexity $d^{\Omega(p)}$.
\begin{assumption}\label{assumption:rdim}
	There exists a function $g: \R^r \to \R$ and linearly independent vectors $u_1,\ldots,u_r$ such that for all $x \in \R^d$,
	\begin{align*}
		f^\star(x) = g(\langle x, u_1 \rangle,\ldots,\langle x, u_r \rangle).
	\end{align*}
	We will call $S^\star := \spn(u_1,\ldots,u_r)$ the principal subspace of $f^\star$. We will also denote by $\Pi^\star := \Pi_{S^\star}$ the orthogonal projection onto $S^\star$.
\end{assumption}

Note that \Cref{assumption:rdim} guarantees that for any $x$, $\spn(\nabla^2 f^\star(x)) \subset S^\star$. In particular, if we denote the average Hessian by $H := \E_{x \sim \mathcal{D}}[\nabla^2 f^\star(x)]$, we have that $\spn(H) \subset S^\star$ so that $H$ has rank at most $r$. The following non-degeneracy assumption states that $H$ has rank \textit{exactly} $r$.

\begin{assumption}\label{assumption:C2rank}
	$H := \E_{x \sim \mathcal{D}}[\nabla^2 f^\star(x)]$ has rank $r$, i.e. $\spn(H) = S^\star$.
\end{assumption}
We will also denote the normalized condition number of $H$ by $\kappa := \frac{\norm*{H^\dagger}}{\sqrt{r}}$.

\subsection{The Network and Loss}\label{sec:network}
Let $\sigma(x) = \relu(x) = \max(0,x)$, let $a \in \R^m$, $W \in \R^{m \times d}$, $b \in \R^m$, and let $\theta = (a,W,b)$. We define the neural network $f_\theta$ by
\begin{align*}
	f_\theta(x) &= a^T \sigma(Wx + b) = \sum_{j=1}^m a_j \sigma(w_j \cdot x + b_j),
\end{align*}
where $m$ denotes the width of the network. We use a symmetric initialization, so that $f_{\theta_0}(x) =0$~\citep{chizat2018note}. Explicitly, we will assume that $m$ is an even number and that 
\begin{align*}
a_j = -a_{m-j} \qc w_j = w_{m-j} \qand b_j = b_{m-j} \qquad \forall j \in [m/2].	
\end{align*}
We will use the following initialization:
\begin{align*}
	a_j \sim \qty{-1,1} \qc w_j \sim N\qty(0,\frac{1}{d} I_d) \qand b_j = 0.
\end{align*}
We note that while we focus on such symmetric initialization for clarity of exposition, our results also hold with small random initialization that is not necessarily symmetric. This holds by simple modifications in the proof accounting for the small nonzero output of the network at initialization. We will also denote the empirical and population losses by $\mathcal{L}(\theta)$ and $\mathcal{L}_\mathcal{D}(\theta)$ respectively:
\begin{align*}
	\mathcal{L}(\theta) = \frac{1}{n} \sum_{i=1}^n \left(f(x_i) - y_i\right)^2 \qand \mathcal{L}_\mathcal{D}(\theta) = \E_{x \sim \mathcal{D}}\qty[\left(f(x) - f^\star(x)\right)^2].
\end{align*}

\subsection{Notation}

We use $\lesssim, O(\cdot), \Omega(\cdot)$ to denote quantities that are related by absolute constants and we treat $p,\varsigma = O(1)$. We use $\tilde O, \tilde \Omega$ to hide additional dependencies on $\polylog(mnd)$. We denote the $L^1(\mathcal{D})$ and $L^2(\mathcal{D})$ losses of a function $f$ by $\E_{x,y}\abs{f(x) - y}$ and $\E_{x,y}\qty(f(x) - y)^2$ respectively where $x \sim \mathcal{D}$, $y = f^\star(x) + \epsilon$, and $\epsilon \sim \qty{\pm \varsigma}$.

%% file: results.tex
Before we formally state our main result let us specify the exact form of gradient-based training we use in our theory.
\begin{algorithm}
	\SetKwInOut{Input}{Input}
	\SetKwBlock{Preprocess}{preprocess data}{end}
	\SetKw{Reinitialize}{re-initialize}
	
	\Input{Learning rates $\eta_t$, weight decay $\lambda_t$, number of steps $T$}
	
	\Preprocess{
	$\alpha \leftarrow \frac{1}{n} \sum_{i=1}^n y_i$, $\beta \leftarrow \frac{1}{n} \sum_{i=1}^n y_i x_i$
	
	$y_i \leftarrow y_i - \alpha - \beta \cdot x_i$ for $i=1,\ldots,n$
	}
	$W^{(1)} \leftarrow W^{(0)} - \eta_1 [\nabla_{W} \mathcal{L}(\theta) + \lambda_1 W]$
	
	\Reinitialize{ $b_j \sim N(0,1)$ }
	
	\For{$t=2$ \KwTo $T$}{
		$a^{(t)} \leftarrow a^{(t-1)} - \eta_t [\nabla_{a} \mathcal{L}(\theta^{(t-1)}) + \lambda_t a^{(t-1)}]$
	}
	
	\Return Prediction function $x \to \alpha + \beta \cdot x + a^T \sigma(Wx + b)$
	\caption{Gradient-based training}\label{algorithm:layerwise_training}
\end{algorithm}

\noindent With this algorithm in place, we are now ready to state our main result.
\begin{theorem}\label{thm:sample_complexity} Consider the data model, network and loss per Section \ref{secsetup} and train the network via $\Cref{algorithm:layerwise_training}$ with parameters $\eta_1 = \tilde O(\sqrt{d})$, $\lambda_1 = \eta_1^{-1}$, and $\eta_t = \eta,\lambda_t = \lambda$ for $t \ge 2$. Furthermore, assume $n \ge \tilde \Omega(d^2 \kappa^2 r)$ and $d \ge \tilde\Omega(\kappa r^{3/2})$. Then, there exists $\lambda$ such that if $\eta$ is sufficiently small, $T = \tilde\Theta(\eta^{-1}\lambda^{-1})$ and $\theta^{(T)}$ denotes the final iterate of $\Cref{algorithm:layerwise_training}$, we have that the excess population loss in $L^1(\mathcal{D})$ is bounded with probability at least $0.99$ by
	\begin{align*}
	&\E_{x,y}\abs{f_{\theta^{(T)}}(x) - y} - \varsigma \le \tilde O\qty(\sqrt{\frac{d r^p \kappa^{2p}}{n}} + \sqrt{\frac{r^p \kappa^{2p}}{m}} + \frac{1}{n^{1/4}}).
	\end{align*}
\end{theorem}

It is useful to note that the use of $\lambda$ in the algorithm corresponds to the common practice of weight decay and its value is chosen in such a way that $\norm{a^{(T)}} \le B_a$, i.e. to solve a constrained minimization problem (see \Cref{sec:sketch:samplecomplexity}). In practice, one simply tunes the hyperparameter $\lambda$ in order to achieve the desired tradeoff between training and test loss.

An intriguing aspect of the above result is that despite the fact that $f^\star$ may be of arbitrarily high degree, learning $f^\star$ requires only $n \gtrsim dr^p + d^2 r$ samples and only requires a very small network with $m \gtrsim r^p$. We note that our dependence on the latent dimension $r$ is near optimal as the minimax sample complexity even when the principal subspace $S^\star$ is known is $\Theta(r^p)$. 

We show in \Cref{thm:transferlearning} that by resampling the data after the first step, the sample complexity can be further reduced to $d^2 r + r^p$, dropping a factor of $d$ from the second term. The extra factor of $d$ results from the dependence between the data used in the first and second stages and we believe that a more careful analysis could remove this additional factor.

 We contrast \Cref{thm:sample_complexity} with the following lower bound for learning a function class which satisfies \Cref{assumption:rdim} with $r=1$ but does not satisfy \Cref{assumption:C2rank}.

\begin{theorem}\label{thm:CSQ_lower_bound}
	For any $p \ge 0$, there exists a function class $\mathcal{F}_p$ of polynomials of degree $p$, each of which depends on a single relevant dimension, such that any correlational statistical query learner using $q$ queries requires a tolerance $\tau$ of at most
	\begin{align*}
		\tau \le \frac{\log^{p/4} \qty(qd)}{d^{p/4}}
	\end{align*}
	in order to output a function $f \in \mathcal{F}_p$ with $L^2(\mathcal{D})$ loss at most $1$.
\end{theorem}

Using the heuristic $\tau \approx \frac{1}{\sqrt{n}}$, which represents the expected scale of the concentration error, we get the immediate corollary that violating \Cref{assumption:C2rank} allows us to construct a function class which \textit{any neural network} with polynomially many parameters trained for polynomially many steps of gradient descent cannot learn without at least $n \gtrsim d^{p/2}$ samples. We emphasize that this is only a heuristic argument as concentration errors are random rather than adversarial.

 On the other hand, \Cref{thm:sample_complexity} shows that incorporating \Cref{assumption:C2rank} allows gradient descent to efficiently learn polynomials of arbitrarily high degree with only $d^2 r + d r^p$ samples.

The difference in sample complexity between \Cref{thm:sample_complexity} and \Cref{thm:CSQ_lower_bound} is that in \Cref{thm:sample_complexity}, our non-degeneracy assumption (\Cref{assumption:C2rank}) allows the network $f_\theta$ to extract useful features that aid robust learning and allowed learning high degree polynomials with $n \gtrsim d^2$ samples. \Cref{thm:CSQ_lower_bound} shows that violating this assumption allows us to construct a function class which cannot be learned without $d^{\Omega(p)}$ samples, demonstrating the necessity of \Cref{assumption:C2rank}.

 The fact that the network $f_\theta$ extracts useful features not only allows it to learn $f^\star$ efficiently, but also allows for efficient \textit{transfer learning}. In particular, \Cref{thm:transferlearning} shows that we can efficiently learn any target polynomial $g^\star(x)$ that depends on the same relevant dimensions as $f^\star$ with sample complexity \textit{independent of $d$} by simply truncating and retraining the head of the network:

\begin{theorem}\label{thm:transferlearning}
	Let $g^\star(x)$ be a degree $p$ polynomial with $\E_\mathcal{D}[g^\star(x)^2] = 1$ and $g(x) = g(\Pi^\star x)$ for all $x \in \R^d$. Let $D_N = \qty{(x_i,y_i)}_{i \in [N]}$ be a second dataset with $y_i = g(x_i) + \epsilon_i$. We retrain the last layer of the network $f_\theta$ in \Cref{thm:sample_complexity} with gradient descent with learning rate $\eta$ and weight decay $\lambda$, i.e. we will use the function class:
	\begin{align*}
		g_a(x) = a^T(W^{(1)}x + b)
	\end{align*}
	where $W^{(1)}$ is the second iterate of \Cref{algorithm:layerwise_training} on the pre-training dataset. Assume that $d = \tilde\Omega(\kappa r^{3/2})$. Then there exists $\lambda$ such that if the network is pretrained on $n \ge \tilde \Omega(d^2 \kappa^2 r)$ datapoints from $f^\star$ and $\eta$ is sufficiently small, the excess population loss in $L^1(\mathcal{D})$ after $T=\tilde\Theta(\eta^{-1}\lambda^{-1})$ steps is bounded with probability at least $0.99$ by
	\begin{align*}
		\E_{x,y}\abs{g_{a^{(T)}}(x) - y} - \varsigma \le \tilde O\qty(\sqrt{\frac{r^{p}\kappa^{2p}}{\min(m,N)}} + \frac{1}{N^{1/4}}).
	\end{align*}
\end{theorem}
Learning $g^\star(x)$ therefore only requires $N,m \gtrsim r^p$, which is independent of the ambient dimension $d$. We note that this is minimax optimal for learning arbitrary degree $p$ polynomials even when the hidden subspace $S^\star$ is known. \Cref{thm:transferlearning} also shows that $n \gtrsim d^2 r$ pre-training samples are sufficient for gradient descent to learn the subspace $S^\star$ from the pre-training data.

%% file: related_work.tex
 A growing body of recent work show the connection between gradient descent on the full network and the Neural Tangent Kernel (NTK) \cite{jacot2018neural,oymak2019overparameterized,oymak2020towards,du2019gradient,arora2019fine, du2018gradient, lee2019wide}. Using this technique one can prove concrete results about neural network training~\citep{li2018learning,du2018gradient, du2019gradient,allen2018convergence,zou2018stochastic} and generalization~\citep{arora2019fine,oymak2019generalization,allen2019learning,cao2019generalization, oymak2021generalization} in the kernel regime. The key idea is that for a large enough initialization, it suffices to consider a linearization of the neural network around the origin. This allows connecting the analysis of neural networks with the well-studied theory of kernel methods. This is also sometimes referred to as \textit{lazy training}, as with such an initialization the parameters of the neural networks stay close to the parameters at initialization and these results can only show that neural networks are as powerful as shallow learners such as kernels. There is however growing evidence that this NTK-style analysis might not be sufficient to completely explain the success of neural networks in practice. The papers \cite{chizat2019lazy,woodworth2019kernel} provides  empirical evidence that by choosing a smaller initialization the test error of the neural network decreases. A similar performance gap between the performance of the  NTK and neural networks has been observed in \cite{ghorbani2020neural}. This NTK-style analysis however does not yield satisfactory results in the setting studied in this paper. In particular for learning the polynomials of the form we study in this paper, \cite{ghorbani2019linearized} demonstrates that one needs at least $d^p$ samples in the kernel regime. In contrast, our results only require on the order of $d^2$ samples.

Leveraging the fact that linearized models are not feature learners, 
\citet{ghorbani2019linearized} and \cite{wei2019regularization} showed precise upper and lower bounds on the sample complexity of NTK methods. They showed that because NTK is unable to learn new features, learning any polynomial in dimension $d$ of degree $p$ requires $n = \Theta(d^p)$ samples, which \textit{gives no improvement} over polynomial kernels. 
On the empirical front, the NTK linearization analysis is also lacking. \citet{arora2019exact} demonstrated that the kernel predictor loses more than $20\%$ in test accuracy relative to a deep network trained with SGD and state-of-art regularization on CIFAR-10. Our work is motivated by the contrast between these negative theoretical results for linearized NTK models and the spectacular empirical performance of deep learning.

The gap between such shallow learners and the full neural network has been established in theory ~\citep{wei2019regularization,allen2020backward,allen2019can,yehudai2019power,ghorbani2019limitations,woodworth2020kernel,dyer2019asymptotics,du2018power} and observed in practice~\citep{arora2019exact,lee2019wide,chizat2018note}.  There is an emerging literature on learning beyond the lazy/NTK regime in the small initialization setting. The papers \cite{li2018algorithmic,stoger2021small} shows that for the problem of low-rank reconstruction in a non-lazy regime with small random initialization gradient descent finds globally optimal solutions with good generalization capability. This is carried out by utilizing a spectral bias phenomena exhibited by the early stages of gradient descent from small random initialization that puts the iterates on the trajectory towards generalizable models. For the problem of  tensor decomposition it has also been shown that gradient descent with small initialization is able to leverage low-rank structure \citep{wang2020beyondNTK}. In \cite{li2020learningbeyondNTK},  it has been shown that neural networks with orthogonal weights  can be learned via SGD and outperform any kernel method. One crucial element in their analysis is that the early stage of the training is connected with learning the first and second moment of the data. Higher-order approximations of the training dynamics~\citep{bai2019beyond,bai2020taylorized} and the Neural Tangent Hierarchy~\citep{huang2019dynamics} have also been recently proposed towards closing this gap. None of the above papers, however, focus on learning polynomial representations efficiently via neural networks as carried out in this paper.

 Another line of work focuses on learning single activations such as the ReLU function. In this context \citep{YS19} shows that it is hard to learn a single ReLU activation via stochastic gradient descent with random features where as learning such activations is possible in a non-NTK regime \citep{Mahdi17,GKKT17,GKK19-neurips} again highlighting this important gap. In related work where the label also only depends on a single relevant direction~\citep{daniely2020learning}, the authors show that in the context of learning the parity function, gradient descent is able to efficiently learn the planted set. However, this is a result of the unbalanced data distribution which skews the gradient towards the planted set. In contrast, we consider isotropic Gaussian data so that no information can be extracted from the data distribution itself and features must be extracted from higher order correlations between the data and the labels. \citet{chen2020learning} also studied the problem of learning polynomials of few relevant dimensions. They provide an algorithm that learns polynomials of degree $p$ in $d$ dimensions that depends on $r$ hidden dimensions with $n \gtrsim C(r,p) d$ samples where $C(r,p)$ is an unspecified function of $r,p$ which is likely exponential in $r$. However, their algorithm is not a variant of gradient descent, and requires a clever spectral initialization. On the other hand, this work focuses on the ability of gradient descent to \textit{automatically} extract hidden features and learn representations from the data.

There is also a line of work \cite{mei2018mean,chizat_meanfield,mei2019mean,javanmard2020analysis,sirignano,wei2019regularization}, which is concerned with the mean-field analysis of neural networks. The insight is that for sufficiently large width the training dynamics of the neural network can be coupled with the evolution of a probability distribution described by a PDE. These papers use a smaller initialization than in the NTK-regime and, hence, the parameters can move away from the initialization. However, these results do not provide explicit convergence rates and require an unrealistically large width of the neural network. To the extent of our knowledge such an analysis technique has not been used to show efficient learning of polynomial representations using neural networks as carried out in this paper.

A concurrent line of work studied the feature learning ability of gradient descent in the mean field regime with data sampled from the boolean cube \citep{abbe2022merged}. The authors identified a necessary and sufficient condition for learning with sample complexity linear in $d$, dubbed the \textit{merged staircase property}, in the special case when the hidden weights of the two layer neural network are initialized at $0$. However, the zero initialization hinders the feature learning ability of the network. For example, the boolean function XOR violates the merged staircase property, however noisy XOR is known to be learnable by two layer neural networks with sample complexity linear in $d$ \citep{bai2019beyond, chen2020towardshierarchical}. In this work we study the impact that the nonzero initialization of the hidden weights has on the feature learning ability of neural networks.

~\\

%\noindent\textbf{Linear neural networks:} In \cite{bartlett2018gradient,arora2018optimization,arora2018convergence,bah2019learning,chou2020gradient}  the convergence of gradient flow and gradient descent is studied for (deep) linear neural networks of the form
%\begin{equation*}
%\underset{W_1, W_2, \ldots, W_N}{\min} \  \sum_{i=1}^m \big\Vert W_N \ldots W_2 W_1 x_i - y_i   \big\Vert^2.
%\end{equation*}
%However, note that this model is different from the one studied in this paper. In \cite{gunasekar2018implicit} it is shown that gradient descent for convolutional linear neural networks has a bias towards the $\ell_p$-norm, where $p$ depends on the depth of the network.  

%% file: sketch.tex
\subsection{Proof of \Cref{thm:sample_complexity,thm:transferlearning}}\label{sec:sketch:samplecomplexity}
The proofs of \Cref{thm:sample_complexity,thm:transferlearning} are essentially identical so we will focus on \Cref{thm:sample_complexity}. We begin by noting that the symmetric initialization implies that $f_\theta(x) = 0$ for all $x \in \R^d$. This implies that the population gradient of each feature $w_j$ can be written as
\begin{align*}
	\nabla_{w_j} \mathcal{L}_\mathcal{D}(\theta) = \E_{x \sim \mathcal{D}}\qty[2(f_\theta(x)-f^\star(x))\nabla_{w_j} f_\theta(x)] = -2\E_{x \sim \mathcal{D}}[f^\star(x)\nabla_{w_j} f_\theta(x)].
\end{align*}
Using the chain rule, we can further expand this as
\begin{align*}
	-2\E_{x \sim \mathcal{D}}[f^\star(x)\nabla_{w_j} f_\theta(x)] = -2a_j \E_x[f^\star(x) x \1_{w_j \cdot x \ge 0}].
\end{align*}
The main computation that drives \Cref{thm:sample_complexity,thm:transferlearning} is that for any unit vector $w \in \R^d$, the expression $\E_x[f^\star(x) x \1_{w \cdot x \ge 0}]$ has a natural series expansion in powers of $w$, which can be computed explicitly in terms of the Hermite expansions of $f^\star$ and $\sigma'$. Explicitly, if $C_k = \E_x[\nabla^k f^\star(x)]$ is a symmetric $k$ tensor denoting the expected $k$th derivative of $f^\star$ and $c_k$ are the Hermite coefficients of $\sigma'(x) = \1_{x \ge 0}$,
\begin{align}
	\E_x[f^\star(x) x \1_{w \cdot x + b \ge 0}] \nonumber
	&= \sum_{k=0}^{p-1} \frac{1}{k!}\left[c_{k+1} C_{k+1}(w^{\otimes k}) + c_{k+2} w C_k(w^{\otimes k})\right] \nonumber \\
	&= \underbrace{\frac{H w}{\sqrt{2\pi}}}_{O(d^{-1/2})} + \underbrace{\frac{1}{2}\qty[c_3 C_3(w,w) + c_4 w C_2(w,w)]}_{O(d^{-1})} + \underbrace{\frac{1}{6}\qty[\cdots]}_{O(d^{-3/2})} + \ldots \label{eq:gradient_series}
\end{align}
where we note that $C_2 = \E_x[\nabla^2 f^\star(x)] = H$. We emphasize that because $w$ is a unit vector, its inner product with any fixed unit vector is of order $d^{-1/2}$ so temporarily ignoring factors of $r$, $\norm{C_{k+1}(w^{\otimes k})},\abs{C_k(w^{\otimes k})} = O(d^{-\frac{k}{2}})$. Therefore \Cref{eq:gradient_series} is an asymptotic series in $d^{-1/2}$. As $k$ increases, each term in \Cref{eq:gradient_series} reveals more information about $f^\star$. However, this information is also better hidden. A standard concentration argument shows that extracting information from the $C_k$ term in this series requires $n \ge d^k$ samples. This paper focuses on the first term in this expansion, $\frac{H w}{\sqrt{2\pi}}$, which requires $n \ge d^2$ samples to isolate. We directly truncate this series expansion:
\begin{lemma}\label{lem:feature_error_bound}
	With high probability over the random initialization,
	\begin{align*}	
		\nabla_{w_j} \mathcal{L}_\mathcal{D}(\theta) = -2a_j\frac{Hw_j}{\sqrt{2\pi}} + \tilde O\qty(\frac{\sqrt{r}}{d}).
	\end{align*}
\end{lemma}
Note that the remainder term, of order $d^{-1}$, contains all higher order terms in the series expansion. 

Recall that $H = \E_{x \sim \mathcal{D}}[\nabla^2 f^\star(x)]$ is the average Hessian of $f^\star$ with respect to $\mathcal{D}$. Because $f^\star$ depends only on the subspace $S^\star$, this implies that up to higher order terms, the population gradient at initialization already points each feature vector $w_j$ towards the principal subspace $S^\star$. In addition, \Cref{assumption:C2rank} guarantees that the gradients at initialization span the principal subspace $S^\star$. 

However, it is also important to note that the population gradient is bounded by $\norm{\nabla_{w_j} \mathcal{L}_\mathcal{D}(\theta)} = O(d^{-1/2})$ and we only have access to the empirical gradient $\nabla_{w_j} \mathcal{L}(\theta)$. As mentioned above, extracting the necessary subspace information from $\nabla_{w_j} \mathcal{L}_\mathcal{D}(\theta)$ to learn $f^\star$ therefore requires $n \gtrsim d^2$ samples, which is the dominant term in our final sample complexity result.

Once we show that the gradient at initialization contains all the relevant features, we note that after the first step of gradient descent,
\begin{align*}
	W^{(1)} = W^{(0)} - \eta_1[\nabla_W \mathcal{L}(\theta^{(0)}) + \eta_1^{-1} W] = -\eta_1 \nabla_W \mathcal{L}(\theta^{(0)}).
\end{align*}
After the first step, the model therefore resembles a random feature model with random features $\{Hw\}_{w \in S^{d-1}} \subset S^\star$. Previous results have shown that in these linearized regimes, e.g. random feature models/NTK, learning degree $p$ polynomials requires $n \gtrsim d^p$ samples and width $m \gtrsim d^{p}$. As our ``random features'' are now constrained to the hidden subspace $S^\star$, which has dimension $r$, we should expect that our sample complexity improves to $n \gtrsim r^p$.

The remainder of \Cref{algorithm:layerwise_training} runs ridge regression on the network head $a$ with fixed features $x \to \sigma(W^{(1)}x + b)$. We can directly analyze the generalization of this algorithm using standard techniques from Rademacher complexity. In particular, a high level sketch of the remainder of the proof goes as follows:
\begin{enumerate}
	\item (\Cref{sec:appendix:random_feature_approx}): We use the features from \Cref{lem:feature_error_bound} to construct a vector $a^\star \in \R^m$ such that
	\begin{align*}
		\mathcal{L}(a^\star,W^{(1)},b) \ll 1 \qand \norm{a^\star} = \tilde O\qty(\frac{r^p \kappa^{2p}}{\sqrt{m}}).
	\end{align*}
	
	\item (\Cref{sec:thm1proof}): We show the equivalence between ridge regression and norm constrained linear regression implies the existence of $\lambda > 0$ such that the $T$th iterate $a^{(T)}$ satisfies
	\begin{align*}
		\mathcal{L}(a^{(T)},W^{(1)},b) \ll 1 \qand \norm*{a^{(T)}} \le \norm{a^\star}.
	\end{align*}
	
	\item (\Cref{sec:thm1proof}): A standard Rademacher generalization bound for two layer neural networks bounds the population risk $\E_{x,y} \abs{f_{\theta^{(T)}}(x)-y}$ by the empirical risk $\frac{1}{n} \sum_{i=1}^n \abs{f_{\theta^{(T)}}(x_i) - y_i}$ and $\|a^{(T)}\|$ which are small from step 2.
\end{enumerate}

\subsection{Proof of \Cref{thm:CSQ_lower_bound}}\label{sec:sketch:CSQ}

Statistical query learners are a family of learners that can query values $q(x,y)$ and receive outputs $\hat q$ with $\abs{\hat q - \E_{x,y}[q(x,y)]} \le \tau$ where $\tau$ denotes the query tolerance \citep{goel2020superpolynomial,diakonikolas2020algorithms}. An important class of statistical query learners is that of correlational/inner product statistical queries (CSQ) of the form $q(x,y) = yh(x)$. This includes a wide class of algorithms including gradient descent with square loss. For example from \Cref{sec:sketch:samplecomplexity}, for a two layer neural network we have
\begin{align*}
	\nabla_{w_j} \mathcal{L}_\mathcal{D}(\theta) = \E_{x,y}[yh(x)] \qq{where} h(x) = -2a_j x \1_{w_j \cdot x + b_j \ge 0}.
\end{align*}

In order to prove \Cref{thm:CSQ_lower_bound}, we must construct a function class $\mathcal{F}_p$ such that inner product queries of the form $\E_{x,y}[yh(x)]$ provide little to no information about the target function. The standard approach is to construct a function class with small pairwise correlations, i.e. for $f \ne g \in \mathcal{F}_p$, $\abs{\E_x[f(x)g(x)]} \le \epsilon$ \citep{goel2020superpolynomial,diakonikolas2020algorithms}. The number of functions in the function class $\mathcal{F}_p$ and the size of the pairwise correlations $\epsilon$ directly imply a correlational statistical query lower bound:
\begin{lemma}[Modified from Theorem 2 in \citet{Szrnyi2009CharacterizingSQ}]\label{lem:general_CSQ}
	Let $\mathcal{F}$ be a class of functions and $\mathcal{D}$ be a data distribution such that 
	\begin{align*}
		\E_{x \sim \mathcal{D}}[f(x)^2] = 1 \qand \abs{\E_{x \sim \mathcal{D}}[f(x) g(x)]} \le \epsilon \qquad \forall f \ne g \in \mathcal{F}.
	\end{align*}
	Then any correlational statistical query learner requires at least $\frac{\abs{\mathcal{F}}(\tau^2-\epsilon)}{2}$ queries of tolerance $\tau$ to output a function in $\mathcal{F}$ with $L^2(\mathcal{D})$ loss at most $2-2\epsilon$.
\end{lemma}
To construct $\mathcal{F}_p$, we begin by showing that there are a large number of approximately orthogonal unit vectors in $S^{d-1}$:
\begin{lemma}\label{lem:approx_orthogonal_vectors}
	There exists an absolute constant $c$ such that for any $\epsilon > 0$, there exists a set $S$ of $\frac{1}{2} e^{c\epsilon^2d}$ unit vectors such that for any $v,w \in S$ such that $v \ne w$, we have $\abs{v \cdot w} \le \epsilon$.
\end{lemma}
The proof bounds the probability that randomly sampled unit vectors have a large inner product and existence then follows from the probabalistic method. Therefore for any $m$, we can find $m$ unit vectors in $\R^d$ such that their pairwise inner products are all bounded by $d^{-1/2}\sqrt{\log m}$. We combine this with the fact that if $f_u(x) = \frac{He_k(u \cdot x)}{\sqrt{k!}}$ where $He_k$ denotes the $k$th Hermite polynomial,
\begin{align*}
	\E_{x \sim \mathcal{D}}\qty[f_u(x)f_v(x)] = (u \cdot v)^k.
\end{align*}
Therefore $\abs{u \cdot v} \le d^{-1/2}\sqrt{\log m}$ implies $\abs{E_{x \sim \mathcal{D}}[f_u(x)f_v(x)]} \le d^{-k/2}(\log m)^{k/2}$. \Cref{thm:CSQ_lower_bound} then directly follows from \Cref{lem:general_CSQ} (see \Cref{sec:appendix:CSQproof} for a more detailed proof).

%% file: experiments.tex
\subsection{Sample Complexity}\label{sec:experiment:sample}

In this section we present a toy example that clearly demonstrates the gap between kernel methods and gradient descent on two layer networks. For $u \in S^{d-1}$, consider the target function
\begin{align}
	f^\star_u(x) = g(u \cdot x) \qq{where} g(x) = \frac{He_2(x)}{2} + \frac{He_p(x)}{\sqrt{2p!}},\label{eq:experiment_fn}
\end{align}
which satisfies $E_{x \sim \mathcal{D}}[f_u^\star(x)^2] = 1$. Note that $f^\star$ only depends on the projection of $x$ onto a single relevant direction, $u$. We show in \Cref{sec:sketch:samplecomplexity} that gradient descent is capable of isolating the subspace spanned by $u$ and then fitting a one dimensional random feature model to $g$, and that this entire process requires $n \asymp d^2$ samples to generalize.

On the other hand, existing works \citet{ghorbani2019linearized,ghorbani2020neural} have shown that $n \asymp d^p$ samples are strictly necessary in order to learn $f^\star$ in the NTK or random features regime. The theory predicts that with $n < d^2$ samples, kernel regression will return the $0$ predictor and with $d^2 < n < d^p$ samples, kernel regression will return $\frac{1}{2}He_2(u \cdot x)$, incurring a $L^2(\mathcal{D})$ loss of $\frac{1}{2}$.

We empirically verify these predictions. We take $d=10$ and $p = 4$ and consider the function $f_{e_1}^\star(x) = \frac{He_2(x_1)}{2} + \frac{He_4(x_1)}{4\sqrt{3}}$. We use label noise $\sigma^2 = 1$ and attempt to learn $f^\star$ using \Cref{algorithm:layerwise_training}, a random feature model, and a linearized NTK model. All experiments are conducted on a two layer neural network with widths $m=100$ and $m=1000$. For each value of $n$, the weight decay parameter $\lambda$ is tuned on a holdout set of size $10^5$ and test accuracies are reported over a separate test set of size $10^5$. Errors bars reflect the mean and standard deviation over $10$ random seeds.

We note that while \Cref{algorithm:layerwise_training} easily converged to vanishing excess risk, even at width $m=100$, both the random features model and the neural tangent kernel model only managed to fit the quadratic term $\frac{1}{2} He_2(u \cdot x)$, as predicted by the theory in \citet{ghorbani2019linearized,ghorbani2020neural}.

The key to learning a function of the form $f^\star_u$ is to use the fact that the $\frac{1}{2}He_2(u \cdot x)$ component of $f_u^\star$ gives enough information to identify $u$. Afterwards, any random feature or kernel method can efficiently fit any sufficiently smooth univariate function $g: \R \to \R$ ontop of $u \cdot x$. Our analysis in \Cref{sec:sketch:samplecomplexity} shows that this is exactly the way that \Cref{algorithm:layerwise_training} learns $f_u^\star$ and this is reflected by the steep and sudden drop from trivial risk ($L^2 = 1$) to vanishing excess risk without plateauing at $L^2 = 0.5$ in \Cref{fig:experiments}.

\subsection{Transfer Learning}

\begin{figure}
	\centering
	\includegraphics[width=0.49\textwidth]{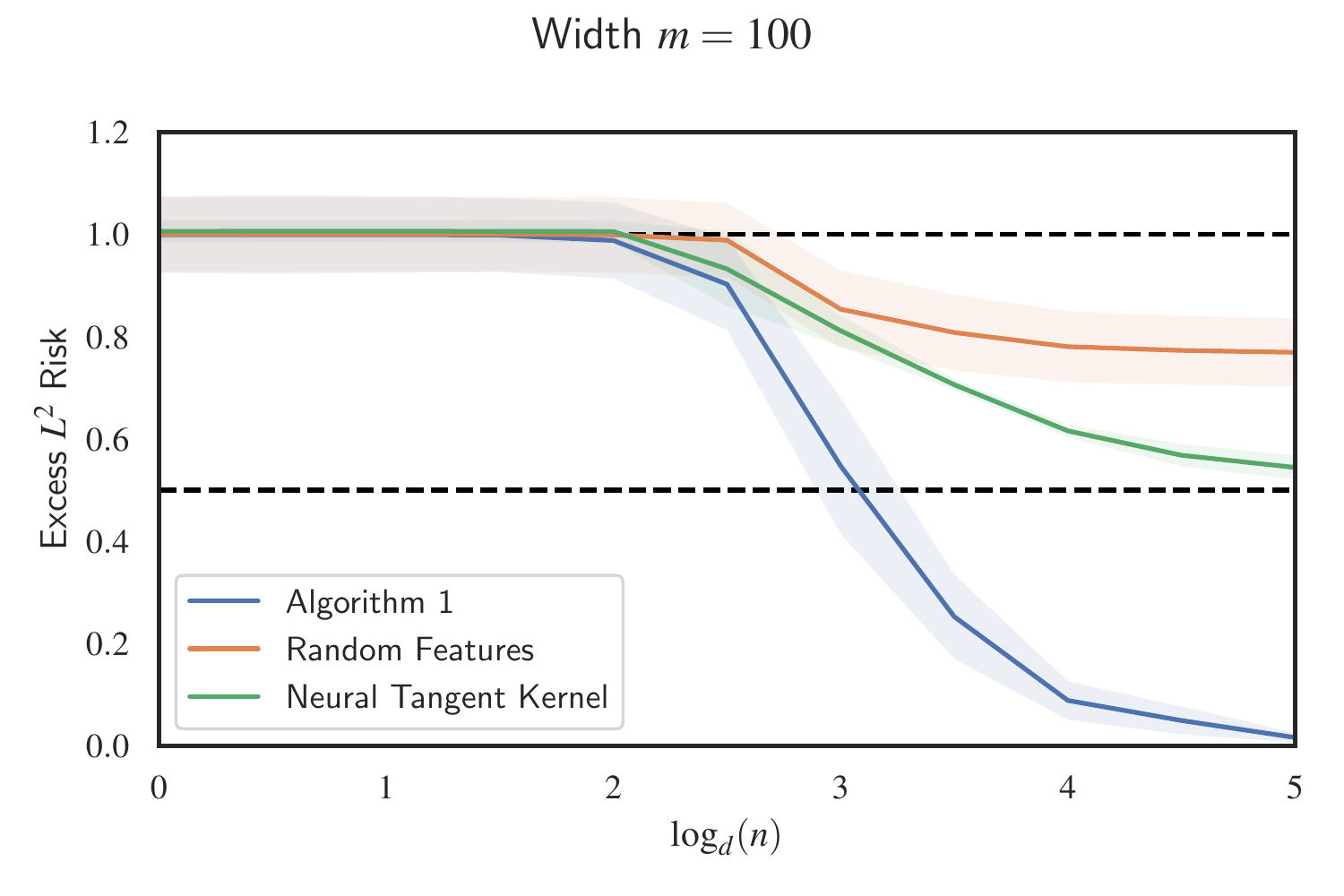}
	\includegraphics[width=0.49\textwidth]{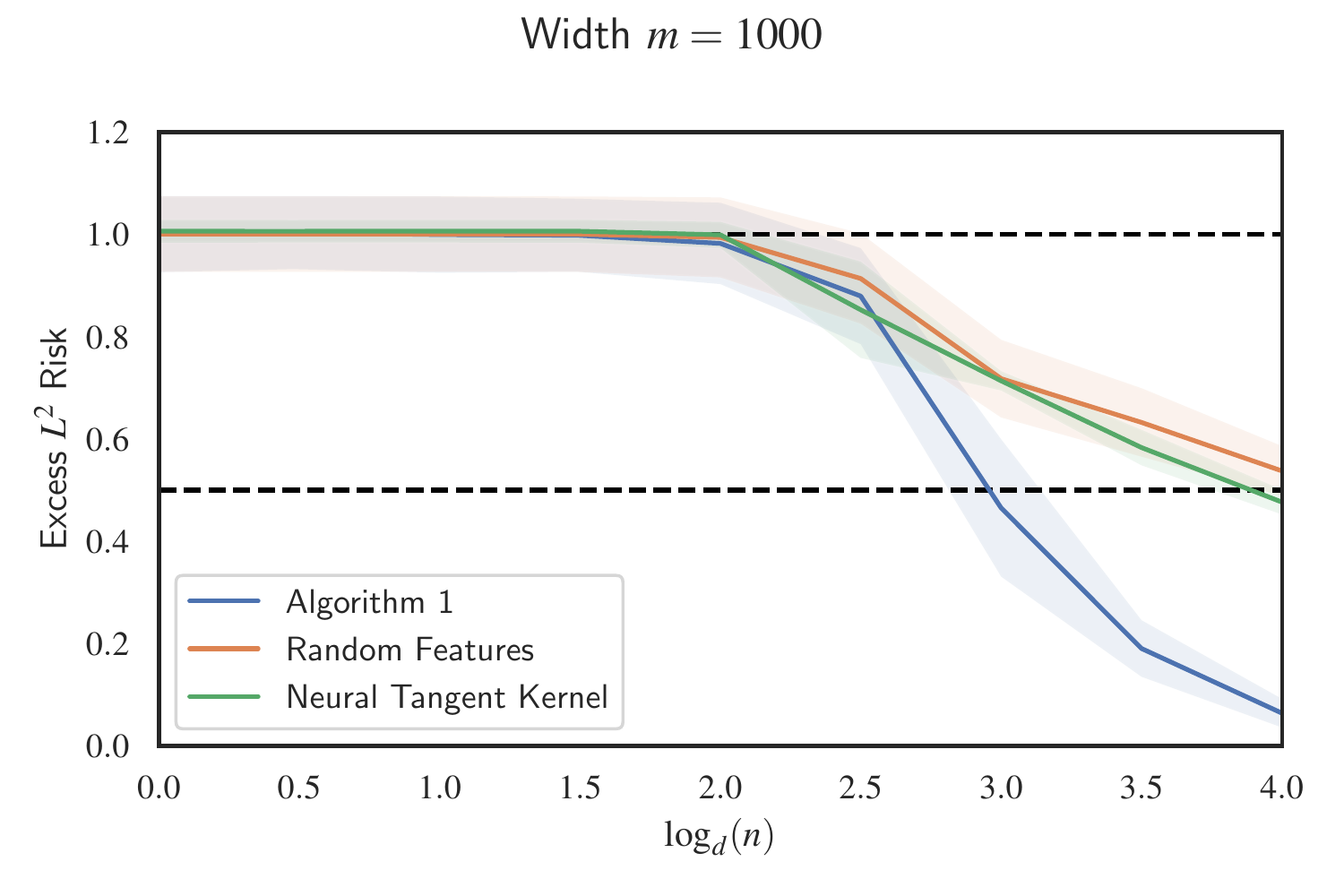}
	\caption{\textbf{Sample Complexity:} The $x$ axis plots $\log_d(n)$ and the $y$ axis plots the excess risk in $L^2(\mathcal{D})$ for each of three methods: \Cref{algorithm:layerwise_training}, random features, and the neural tangent kernel. The top dashed horizontal line is at $L^2 = 1$, which corresponds to outputting the zero predictor. The middle horizontal line is at $L^2 = \frac{1}{2}$ and corresponds to learning the optimal quadratic predictor $\frac{1}{2}He_2(x_1)$. Due to its improved sample efficiency, \Cref{algorithm:layerwise_training} easily achieves near zero excess risk despite the relatively high degree of $f^\star$, while random features and the neural tangent kernel are only able to learn the optimal quadratic predictor. See \Cref{section:experiments} for additional experimental details.}
	\label{fig:experiments}
\end{figure}

The proof of \Cref{thm:sample_complexity} involves showing that \Cref{algorithm:layerwise_training} learns features corresponding to $S^\star$ (see \Cref{sec:sketch:samplecomplexity}) and the proof of \Cref{thm:transferlearning} shows that this implies efficient transfer learning. We again verify this empirically. We consider the function:
\begin{align*}
	f^\star_{\text{target}}(x) = g_{\text{target}}(u \cdot x) \qq{where} g_{\text{target}} = \frac{He_p(x)}{\sqrt{p!}}.
\end{align*}
Note that this was exactly the hard example in \Cref{thm:CSQ_lower_bound} that was unlearnable without $n \gtrsim d^{\frac{p}{2}}$ samples by a correlational statistical query learner (and in particular, gradient-based learners).

We pretrain with $n$ samples on the $f^\star(x)$ from \Cref{sec:experiment:sample}, then train the output layer using $N$ samples from $f^\star_{\text{target}}$. As in \Cref{sec:experiment:sample}, we use a label noise strength of $\sigma^2 = 1$. We pick $p=3$ so that random feature methods or the neural tangent kernel will require at least $n \gtrsim d^3$ samples to learn $f^\star$.

We note that in \Cref{fig:transfer}, when $n=d^0,d^1$, fine tuning on $N$ target samples gives trivial risk until $N \gtrsim d^3$, which is to be expected of a kernel method with no prior information. However, for $n \ge d^2$ pretraining samples, we can fine tune on just $N = O(1)$ target samples to reach nontrivial loss and the loss decays rapidly as a function of $N$. This experiment therefore fully supports the conclusion of \Cref{thm:transferlearning}.

\begin{figure}[h!]
	\centering
	\includegraphics[width=0.49\textwidth]{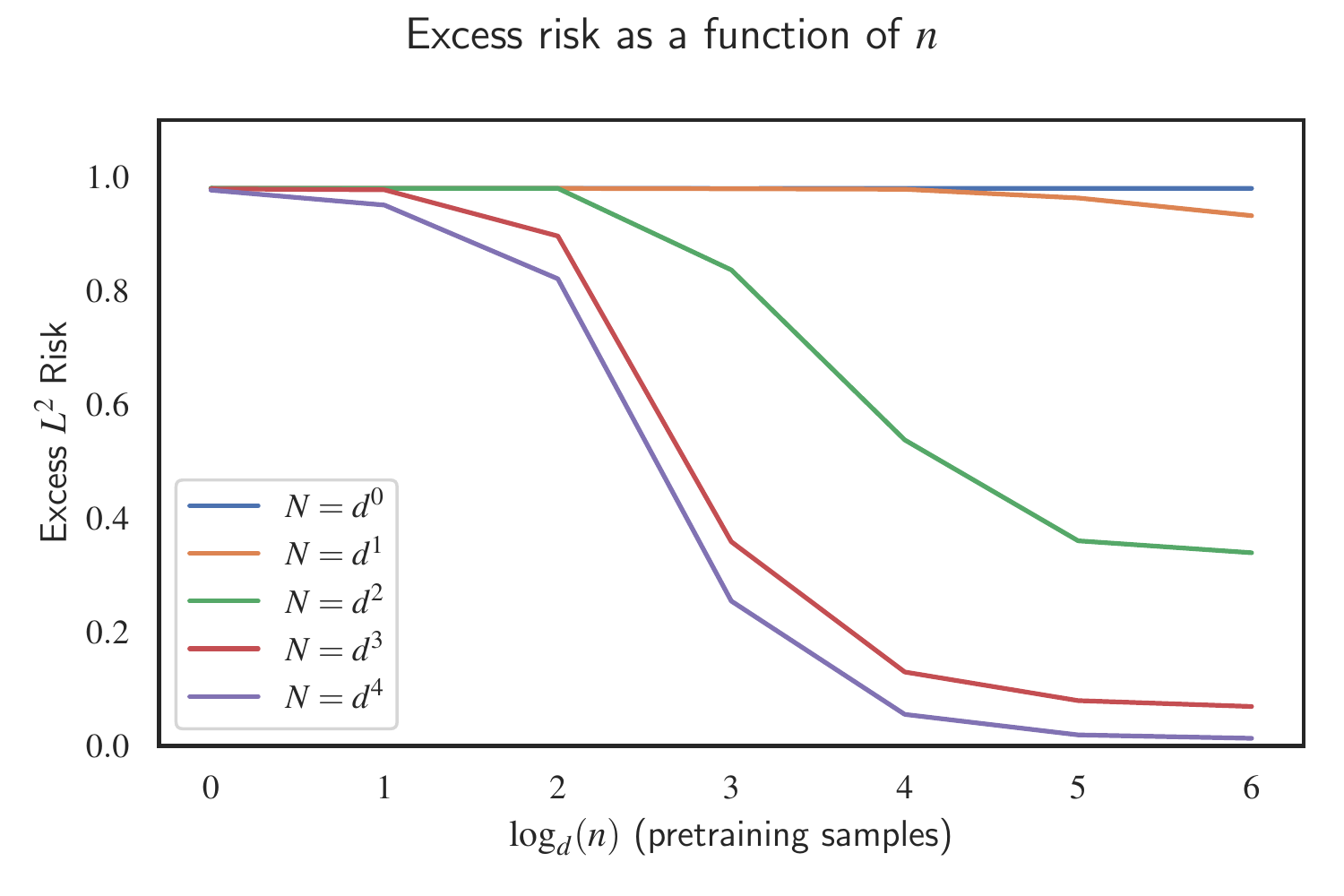}
	\includegraphics[width=0.49\textwidth]{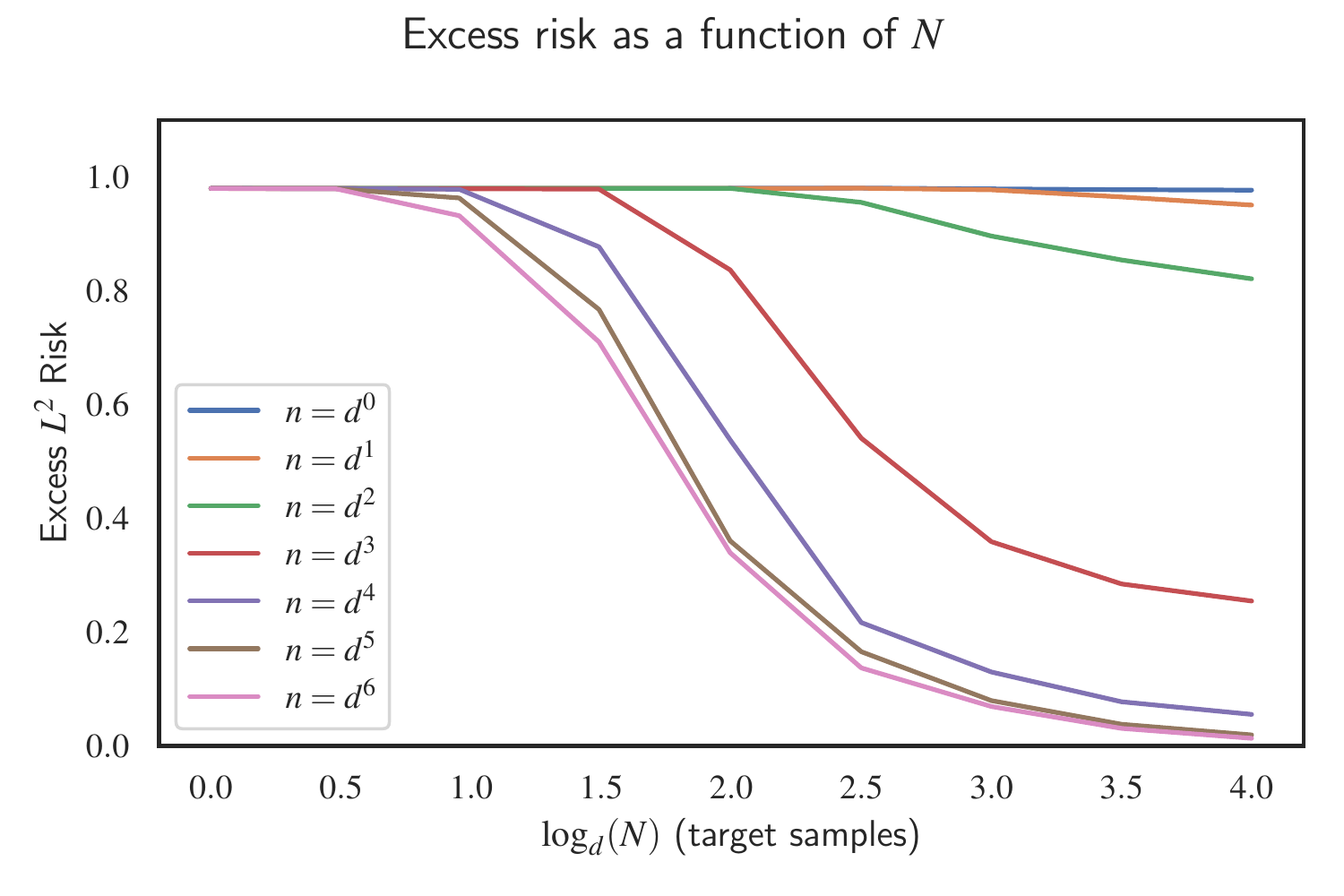}
	\caption{\textbf{Transfer Learning:} The $x$ axes plot $\log_d(n)$ and $\log_d(N)$ respectively. We note that with little pretraining ($\log_d(n) = 0,1$), \Cref{algorithm:layerwise_training} is unable to extract a robust representation that enables transfer. For $n \ge d^2$, we observe that finetuning the representation from \Cref{algorithm:layerwise_training} gives nontrivial loss even for $N = O(1)$, as predicted by \Cref{thm:transferlearning}. See \Cref{section:experiments} for additional experimental details.}
	\label{fig:transfer}
\end{figure}

%% file: discussion.tex
In this work we provide a clear separation between gradient-based training and kernel methods. We show that there is a large family of degree $p$ polynomials which are efficiently learnable by gradient descent with $n \asymp d^2$ samples, in contrast to the lower bound of $d^p$ for random feature/NTK analysis. The main idea driving both our sample complexity result (\Cref{thm:sample_complexity}) and our transfer learning result (\Cref{thm:transferlearning}) is that gradient descent learns useful representations of the data.

One promising direction for future work is tightening the dimension dependence of our upper bound. In particular, our $n \asymp d^2$ sample complexity is driven by the difficult in learning from a degree $2$ Hermite polynomial. However, our lower bound for such functions (\Cref{thm:CSQ_lower_bound}) only rules out learning with $n \le d$ samples. In this situation the lower bound is tight as \citet{chen2020towards} show that sparse degree $2$ polynomials can be efficiently learned with $n \asymp d$ samples.

Another promising direction from future work is generalizing our result to the situation in which the hidden layer and the output layer are trained together. This introduces dependencies between the hidden and output layers which are difficult to control. However, such analysis may lead to a better understanding of learning order and inductive bias in deep learning.

%% file: appendix.tex
\section{Proofs}
We define $\iota = C_\iota \log(nmd)$ for a sufficiently large constant $C_\iota$. Throughout the appendix we will use $e^{-\iota}$ to track failure probabilities of various lemmas and theorems. 

\begin{definition}[High probability events]
	We say that an event $A$ happens \textit{with high probability} if it happens with probability at least $1 - \poly(n,m,d)e^{-\iota}$ where $\poly(n,m,d)$ does not depend on $C_\iota$.
\end{definition}
Note that high probability events are closed under taking union bounds over sets of size $\poly(n,m,d)$. We will assume throughout that $\iota \le c d$ for a sufficiently small absolute constant $c$.

The following lemma bounds $\norm{x_i}$ and is a direct corollary of \Cref{lem:chi_square}:
\begin{lemma}\label{lem:bound_x_norm}
	With high probability, $\norm{x_i}^2 \in \qty[\frac{d}{2},2d]$ for $i=1,\ldots,n$.
\end{lemma}
All remaining proofs will be conditioned on this high probability event.

\subsection{Hermite Expansions}\label{appendix:hermite_expansion}

\subsubsection{Hermite expansion of $\sigma$}
Let $\sigma(x) := \relu(x) = \max(0,x)$. Then the Hermite expansion of $\sigma(x)$ is
\begin{align*}
	\sigma(x) = \frac{1}{\sqrt{2\pi}} + \frac{x}{2} + \frac{1}{\sqrt{2\pi}}\sum_{k \ge 1} \frac{(-1)^{k-1}}{k! 2^k (2k-1)} He_{2k}(x).
\end{align*}
Let $c_k$ denote the Hermite coefficients of $\sigma$, i.e. $\sigma(x) = \sum_{k \ge 0} \frac{c_k}{k!} He_k(x)$. Note that
\begin{align*}
	\sigma'(x) = \sum_{k \ge 0} \frac{c_{k+1}}{k!} He_k(x) = \frac{1}{2} + \frac{1}{\sqrt{2\pi}}\sum_{k \ge 0} \frac{(-1)^k}{k!2^k(2k+1)}He_{2k+1}(x).
\end{align*}
\subsubsection{Hermite Expansion of $f^\star$}
Let the Hermite expansion of $f^\star$ be
\begin{align*}
	f^\star(x) = \sum_{k=0}^p \frac{\langle C_k, He_k(x) \rangle}{k!}
\end{align*}
where $C_k \in (\R^d)^{\otimes k}$ is the symmetric $k$-tensor defined by
\begin{align*}
	C_k := \E_x[\nabla^k f^\star(x)].
\end{align*}
Note that
\begin{align*}
	\nabla f^\star(x) = \sum_{k = 0}^{p-1} \frac{C_{k+1}(He_{k}(x))}{k!} \in \R^d.
\end{align*}
\begin{lemma}[Parseval's Identity]\label{lem:parseval}
\begin{align*}
	1 = \E_x[f^\star(x)^2] = \sum_{k=0}^p \frac{\norm{C_k}_F^2}{k!}.
\end{align*}	
\end{lemma}
Note that as an immediate consequence of \Cref{lem:parseval}, $\norm{C_k}_F^2 \le k!$. In addition, \Cref{assumption:rdim} guarantees that $C_k\qty(x^{\otimes k}) = C_k\qty(\qty(\Pi^\star x)^{\otimes k})$.

\subsubsection{Concentrating $\alpha,\beta$}
\begin{lemma}\label{lem:alpha_beta}
	Let $\alpha = \frac{1}{n} \sum_{i=1}^n y_i$ and $\beta = \frac{1}{n} \sum_{i=1}^n y_i x_i$. Then, with high probability,
	\begin{align*}
		\abs{\alpha-C_0} \lesssim \sqrt{\frac{\iota^p}{n}} \qand \norm{C_1 - \beta} \lesssim \sqrt{\frac{d \iota^{p+1}}{n}}.
	\end{align*}
\end{lemma}
\begin{proof}
	Let $F(x_1,\ldots,x_n) = \frac{1}{n} \sum_{i=1}^n f^\star(x_i) - C_0$. Note that
	\begin{align*}
		\E_{x_1,\ldots,x_n}[F(x)^2] = \frac{1}{n} \Var(f^\star(x)) \le \frac{1}{n}.
	\end{align*}
	The bound on $\abs{\alpha - C_0}$ therefore immediately follows from \Cref{lem:poly_tail} applied to $F$. The bound on $\norm{\beta - C_1}$ is a special case of \Cref{lem:gradient_concentration} with $\sigma(x) = x$.
\end{proof}

\subsubsection{Hermite Expanding the Features}\label{sec:appendix:expand_feature}
Note that by the scale invariance of $\sigma(x) = \relu(x)$, \Cref{algorithm:layerwise_training} does not depend on $\norm{w_j}$ for $j=1,\ldots,m$. Therefore we can assume WLOG that $\norm{w_j} = 1$ for $j=1,\ldots,m$ and $w_j \sim \unif(S^{d-1})$. For the remainder of the appendix we will assume that $\norm{w_j} = 1$.

\begin{definition}
	We define $\widehat f^\star(x) := f^\star(x) - \alpha - \beta \cdot x$.
\end{definition}
The functions $g(w)$ and $g_n(w)$ capture the features that can be learned after one step of gradient descent:
\begin{definition}
	For $\norm{w} = 1$, we define
	\begin{align*}
		g(w) := \E_x[\widehat f^\star(x) x \sigma'(w \cdot x)] \qand g_n(w) := \frac{1}{n} \sum_{i=1}^n \qty(\widehat f^\star(x_i) + \epsilon_i) x_i \sigma'(w \cdot x_i).
	\end{align*}
\end{definition}
We note that $w_j^{(1)} = 2 \eta_1 a_j g_n(w_j)$ and $g(w) = \E_x [g_n(w)]$. In fact, \Cref{cor:full_gradient_concentration} shows that with probability at least $1-4ne^{-\iota}$,
\begin{align*}
	\sup_w \norm{g(w) - g_n(w)} \lesssim \sqrt{\frac{d\iota^{p+1}}{n}}.
\end{align*}

\begin{lemma} With high probability,
	\begin{align*}
		g(w) = \frac{Hw}{\sqrt{2\pi}} + O\qty(\sqrt{\frac{d\iota^{p+1}}{n}} + \sqrt{\frac{r\iota^2}{d^2}}).
	\end{align*}
\end{lemma}
\begin{proof}
	By Stein's lemma and the orthogonality of Hermite polynomials,
\begin{align*}
	g(x) &= \E_{x}[\widehat f^\star(x)x\sigma'(w \cdot x)] \\
	&= \E_x[\nabla \widehat f^\star(x) \sigma'(w \cdot x) + w f^\star(x) \sigma''(w \cdot x)] \\
	&= \sum_{k=0}^{p-1} \frac{c_{k+1}\E_x[\nabla_x^{k+1} \widehat f^\star(x)](w^{\otimes k})}{k!} + w \sum_{k=0}^p \frac{c_{k+2}\E_x[\nabla_x^{k} \widehat f^\star(x)](w^{\otimes k})}{k!}\\
	&= \qty(\frac{C_1 - \beta}{2} + \frac{w (C_0 - \alpha)}{\sqrt{2\pi}}) + \frac{C_2 w}{\sqrt{2\pi}} + \sum_{k \ge 2} \frac{c_{2k} C_{2k}(w^{\otimes 2k-1})}{(2k-1)!} + w \sum_{k \ge 1} \frac{c_{2k+2} C_{2k}(w^{\otimes 2k})}{(2k)!}.
\end{align*}
	Note that these sums are finite as $C_k = 0$ for $k > p$. Next, by \Cref{lem:sphere_power_frobenius_bound} we have the high probability bounds,
	\begin{align*}
		\norm{C_{k+1}(w^{\otimes k})}_F \lesssim \sqrt{\frac{r^{\lfloor \frac{k}{2} \rfloor}\iota^k}{d^k}} \lesssim \sqrt{\frac{r \iota^3 }{d^3}} &\qq{for} k \ge 3 \\
		\norm{w C_{k}(w^{\otimes k})}_F \lesssim \sqrt{\frac{r^{\lfloor \frac{k}{2} \rfloor}\iota^k}{d^k}} \lesssim \sqrt{\frac{r\iota^2}{d^2}} &\qq{for} k \ge 2.
	\end{align*}
	Applying these bounds term by term and using \Cref{lem:alpha_beta} to bound $\abs{C_0 - \alpha}$ and $\norm{C_1 - \beta}$ gives the desired result.
\end{proof}

\begin{corollary}\label{corollary:gn_error_bound} With high probability,
	\begin{align*}
		g_n(w) = \frac{Hw}{\sqrt{2\pi}} + O\qty(\sqrt{\frac{d\iota^{p+1}}{n}} + \sqrt{\frac{r\iota^2}{d^2}}).
	\end{align*}
\end{corollary}

\begin{corollary}\label{corollary:gn_norm}
	With high probability,
	\begin{align*}
		\norm{g_n(w)} \lesssim \sqrt{\frac{\iota^2}{d}} + \sqrt{\frac{d\iota^{p+1}}{n}}.
	\end{align*}
\end{corollary}

Furthermore, it will become necessary to bound terms of the form $g_n(w) \cdot x_i$. Note that $g_n(w)$ and $x_i$ are dependent random variables. The following lemma handles this dependence.
\begin{lemma}\label{lem:gn_xi_concentration}
	Let $w \sim S^{d-1}$ and assume $n \ge d^2 \iota^p$. Then with high probability,
	\begin{align*}
		\max_{j \in [n]} \norm{g_n(w) \cdot x_j} \le \sqrt{\frac{\iota^3}{d}}.
	\end{align*}
\end{lemma}
\begin{proof}
	We can decompose
	\begin{align*}
		\abs{g_n(w) \cdot x_j} &\le \abs{g(w) \cdot x_j} + \abs{[g(w)-g_n(w)] \cdot x_j}.
	\end{align*}
	For the first term, note that $g(w)$ and $x_j$ are independent so $g(w) \cdot x_j \sim N(0,\norm{g(w)}^2)$ so with high probability,
	\begin{align*}
		 \abs{g(w) \cdot x_i} \le \norm{g(w)} \sqrt{2\iota} \lesssim \sqrt{\frac{\iota^3}{d}} + \sqrt{\frac{d\iota^{p+2}}{n}}.
	\end{align*}
	Next,
	\begin{align*}
		&[g(w) - g_n(w)] \cdot x_i \\
		&= x_j \cdot \qty[\frac{1}{n} \sum_{i \ne j} \qty[\widehat f^\star(x_i) x_i \sigma'(w \cdot x_i) - g(w)]] + \frac{1}{n} \qty[\widehat f^\star(x_j) \norm{x_j}^2 \sigma'(w \cdot x_j) - g(w)].
	\end{align*}
	Note that in the first term, the $x_j$ and the sum are independent. Therefore by \Cref{cor:full_gradient_concentration} the first term is bounded with high probability by $O\qty(\sqrt{\frac{d\iota^{p+2}}{n}})$. In addition, by \Cref{lem:poly_tail}, the second term is bounded by $O\qty(\frac{\iota^{p/2} d}{n})$ which completes the proof.
\end{proof}

\subsection{Random Feature Approximation}\label{sec:appendix:random_feature_approx}

\subsubsection{Univariate Random Feature Approximation}
This section shows that after we reinitialize the biases we can use random features to transform the activation $\sigma(x) = \relu(x)$ into $\sigma(x) = x^p$ which is more natural for learning polynomials.

\begin{lemma}\label{lem:1d_rf_uniform}
	Let $a \sim \unif(\qty{-1,1})$, and $b \sim \unif([-1,1])$. Then for any $k \ge 0$ there exists $v_k(a,b)$ such that for $\abs{x} \le 1$,
	\begin{align*}
		\E[v_k(a,b)\sigma(ax + b)] = x^k \qand \sup_{a,b} \abs{v_k(a,b)} \lesssim 1.
	\end{align*}
\end{lemma}
\begin{proof}
	First, for $k=0$ we can take $v_0(a,b) := 6b$. Then,
	\begin{align*}
		\E[v_0(a,b)\sigma(ax+b)]
		&= \frac{3}{2} \int_{-1}^1 b \qty[\sigma(x+b)+\sigma(-x+b)] db \\
		&= \frac{3}{2} \qty[\int_{-x}^1 b(x+b) db + \int_x^1 b(-x+b) db] \\
		&= 1
	\end{align*}
	and $\sup_{a,b} \abs{v_0(a,b)} = 6$. Next, for $k=1$ we can take $v_1(a,b) := 2a$. Then,
	\begin{align*}
		\E[v_1(a,b)\sigma(ax+b)]
		&= \frac{1}{2} \int_{-1}^1 \qty[\sigma(x+b)-\sigma(-x+b)] db \\
		&= \frac{1}{2} \qty[\int_{-x}^1 (x+b) db - \int_x^1 (-x+b) db] \\
		&= x
	\end{align*}
	and we have $\sup_{a,b} \abs{v_1(a,b)} = 2$. Next, note that by integration by parts we have for any function $f$,
	\begin{align*}
		\E[2(1-a)f''(b)\sigma(ax+b)]
		&= \int_{x}^1 f''(b)(-x+b) db \\
		&= f'(1)(-x+1) - \int_x^1 f'(b) \\
		&= f(x) + f'(1)(-x+1) - f(1) \\
		&= f(x) + [f'(1) - f(1)] - f'(1) x.
	\end{align*}
	Therefore for $k \ge 2$ if $f(x) = x^k$ and
	\begin{align*}
		v_k(a,b) := 2(1-a)f''(b) - [f'(1) - f(1)]v_0(a,b) + f'(1) v_1(a,b)
	\end{align*}
	we have
	\begin{align*}
		\E[v_k(a,b)\sigma(ax+b) = x^k \qand \sup_{a,b} \abs{v_k(a,b)} \lesssim 1.
	\end{align*}
\end{proof}
\begin{corollary}\label{corollary:v_construction}
	Let $a \sim \unif(\qty{-1,1})$, and $b \sim N(0,1)$. Then for any $k \ge 0$ there exists $v_k(a,b)$ such that for $\abs{x} \le 1$,
	\begin{align*}
		\E[v_k(a,b)\sigma(ax + b)] = x^k \qand \sup_{a,b} \abs{v_k(a,b)} \lesssim 1.
	\end{align*}
\end{corollary}
\begin{proof}
	Let $\overline v_k$ be the function constructed in \Cref{lem:1d_rf_uniform} and let
	\begin{align*}
		v_k(a,b) = \1_{\abs{b} \le 1} \frac{\overline v_k(a,b)}{2\mu(b)}
	\end{align*}
	where $\mu(b) := \frac{e^{-\frac{x^2}{2}}}{\sqrt{2\pi}}$ denotes the density of $b$. Then,
	\begin{align*}
		\E_{a,b}[v_k(a,b)\sigma(ax+b)]
		&= \E_a \qty[\int_b v_k(a,b) \sigma(ax+b) \mu(b) db] \\
		&= \E_{a,b\sim \unif([-1,1])} [\overline v_k(a,b) \sigma(ax+b)] \\
		&= x^k
	\end{align*}
	and
	\begin{align*}
		\sup_{a,b} \abs{v_k(a,b)}
		&= \sup_{a,b \in [-1,1]} \qty[\frac{\overline v_k(a,b)}{2\mu(b)}] \lesssim 1.
	\end{align*}
\end{proof}

\subsubsection{Multivariable Random Feature Approximation}

\begin{definition} For $\norm{w} = 1$, we define
	\begin{align*}
		r(w) := g_n(w) - \frac{Hw}{\sqrt{2\pi}}.
	\end{align*}
\end{definition}
Recall that \Cref{corollary:gn_error_bound} shows that with high probability, $\norm{r(w)} \lesssim \tilde O\qty(\sqrt{\frac{d}{n}} + \sqrt{\frac{r}{d^2}})$.

\begin{lemma} With high probability over the data $\{x_i\}_{i \in [n]}$, we have for $j \le 4p$,
\begin{align*}
	\E_w\qty[\norm{\Pi^\star r(w)}^j]^{1/j} \lesssim \sqrt{\frac{d\iota^{p+1}}{n}} + \sqrt{\frac{r^2}{d^3}}.
\end{align*}
\end{lemma}
\begin{proof}
	We can decompose $r(w) = \qty[g_n(w) - g(w)] + \qty[g(w) - \frac{Hw}{\sqrt{2\pi}}]$ and note that
	\begin{align*}
		\E_w\qty[\norm{\Pi^\star r(w)}^j]^{1/j}
		&\le \E_w\qty[\norm{\Pi^\star \qty[g_n(w) - g(w)]}^j]^{1/j} + \E_w\qty[\norm{\Pi^\star \qty[g(w) - \frac{Hw}{\sqrt{2\pi}} ]}^j]^{1/j} \\
		&\le \E_w\qty[\norm{\Pi^\star \qty[g(w) - \frac{Hw}{\sqrt{2\pi}} ]}^j]^{1/j} + O\qty(\sqrt{\frac{d\iota^{p+1}}{n}}).
	\end{align*}
	Recall that
	\begin{align*}
		&g(w) - \frac{Hw}{\sqrt{2\pi}} \\
		&= \qty(\frac{C_1 - \beta}{2} + \frac{w (C_0 - \alpha)}{\sqrt{2\pi}}) + \sum_{k \ge 2} \frac{c_{2k} C_{2k}(w^{\otimes 2k-1})}{(2k-1)!} + w \sum_{k \ge 1} \frac{c_{2k+2} C_{2k}(w^{\otimes 2k})}{(2k)!}.
	\end{align*}
	Therefore,
	\begin{align*}
		\Pi^\star \qty[g(w) - \frac{Hw}{\sqrt{2\pi}}] = \sum_{k \ge 2} \frac{c_{2k} C_{2k}(w^{\otimes 2k-1})}{(2k-1)!} + \Pi^\star w \sum_{k \ge 1} \frac{c_{2k+2} C_{2k}(w^{\otimes 2k})}{(2k)!} + O\qty(\frac{d\iota^{p+1}}{n}).
	\end{align*}
	We can bound the $j$th moment term by term. We have by \Cref{cor:tensor_power_r} and \Cref{lem:gaussian_hypercontractivity} that for $k \ge 2$,
	\begin{align*}
		\qty(\E_w \norm{C_{2k}(w^{\otimes 2k-1})}^j)^{1/j} \lesssim \sqrt{\frac{r^{\lfloor\frac{2k-1}{2}\rfloor}}{d^{2k-1}}} \lesssim \sqrt{\frac{r}{d^3}}
	\end{align*}
	and for $k \ge 1$,
	\begin{align*}
	\qty(\E_w \qty[\norm{\Pi^\star w}\abs{C_{2k}(w^{\otimes 2k})}]^j)^{1/j}
	&\lesssim \qty(\E_w \norm{\Pi^\star w}^{2j})^{\frac{1}{2j}} \qty(\E_w \abs{C_{2k}(w^{\otimes 2k})}^{2j})^{\frac{1}{2j}} \\
	&\lesssim \sqrt{\frac{r}{d}} \sqrt{\frac{r^k}{d^{2k}}} \\
	&\lesssim \sqrt{\frac{r^2}{d^3}}.
	\end{align*}
\end{proof}

We can now show that the random features $g_n(w)$ are sufficiently expressive to allow us to efficiently represent any polynomial of degree $p$ restricted to the principal subspace $S^\star$.
\begin{lemma}\label{lem:g_min_sval} For any $k \le p$, there exists an absolute constant $C$ such that if $n \ge Cd^2 r \kappa^2 \iota^{p+1}$ and $d \ge C \kappa r^{3/2}$, 
	\begin{align*}
		\mat\qty(\E\qty[\qty(\Pi^\star g_n(w))^{\otimes 2k}]) \succsim (rd\kappa^2)^{-k} \Pi_{\sym^k(S^\star)}
	\end{align*}
	where $\Pi_{\sym^k(S^\star)}$ denotes the orthogonal projection onto symmetric $k$ tensors restricted to $S^\star$.
\end{lemma}
\begin{proof}
	Note that because every vector in $\spn \mat\qty(\E\qty[\qty(\Pi^\star g_n(w))^{\otimes 2k}])$ is a vectorized symmetric $k$ tensor, it suffices to show that
	\begin{align*}
		\E_{w}[(\Pi^\star g_n(w))^{\otimes 2k}](T,T) \gtrsim (rd\kappa^2)^{-k}
	\end{align*}
	for all symmetric $k$ tensor $T$ with $\norm{T}_F^2 = 1$. Recall that $g_n(w) = \frac{H w}{\sqrt{2\pi}} + r(w)$. Therefore by the binomial theorem,
	\begin{align*}
		\left\langle T, (\Pi^\star g_n(w))^{\otimes k} \right\rangle = \left\langle T, \qty(\frac{Hw}{\sqrt{2\pi}})^{\otimes k} \right\rangle + \delta(w)
	\end{align*}
	where $\abs{\delta(w)} \lesssim \sum_{i=1}^k \norm{T\qty((Hw)^{\otimes k-i})}_F \norm{\Pi^\star r(w)}^i.$ Therefore by Young's inequality,
	\begin{align*}
		\E\qty[\left\langle T, (\Pi^\star g_n(w))^{\otimes k} \right\rangle^2] \ge \E\qty[\frac{1}{2}\left\langle T, \qty(\frac{H w}{\sqrt{2\pi}})^{\otimes k} \right\rangle^2] - \E[\delta(w)^2].
	\end{align*}
	Next by Cauchy-Schwarz,
	\begin{align*}
		\E[\delta(w)^2]
		&\lesssim \sum_{i=1}^k \sqrt{\E_{w}\qty[\norm{T\qty((H w)^{\otimes k-i})}_F^4]\E_{w}\qty[\norm{\Pi^\star r(w)}^{4i}]} \\
		&\lesssim \sum_{i=1}^k \E_{w}\qty[\norm{T\qty((H w)^{\otimes k-i})}_F^2] \sqrt{\E_{w}\qty[\norm{\Pi^\star r(w)}^{4i}]}.
	\end{align*}
	Let $\hat T$ be the symmetric $k$ tensor defined by $\hat T(v_1,\ldots,v_k) = T(H v_1,\ldots,Hv_k)$. Then by \Cref{lem:sphere_power_C2_growth},
	\begin{align*}
		\E\qty[\norm{T((H w)^{\otimes k-i})}_F^2]
		&\le \frac{1}{\lambda_{min}(H)^{2i}} \E\qty[\norm{\hat T(w^{\otimes k-i})}^2] \\
		&\lesssim \qty(\frac{d}{\lambda_{min}(H)^2})^{i} \E\qty[\left\langle T, \qty(\frac{Hw}{\sqrt{2\pi}})^{\otimes k} \right\rangle^2] \\
		&= (rd\kappa^2)^{i}\E\qty[\left\langle T, \qty(\frac{Hw}{\sqrt{2\pi}})^{\otimes k} \right\rangle^2].
		\end{align*}
	Therefore,
	\begin{align*}
		\E_{w}[\delta(w)^2] \lesssim \E\qty[\left\langle T, \qty(\frac{Hw}{\sqrt{2\pi}})^{\otimes k} \right\rangle^2] \sum_{i=1}^k \qty((r d \kappa^2)\qty(\frac{d\iota^{p+1}}{n} + \frac{r^2}{d^3}))^i.
	\end{align*}
	Because we assumed $n \ge Cd^2 r \kappa^2 \iota^{p+1}$ and $d \ge C \kappa r^{3/2}$ for a sufficiently large constant $C$, we have
	\begin{align*}
		\E_{w}[\delta(w)^2] \lesssim \frac{1}{4} \E\qty[\left\langle T, \qty(\frac{Hw}{\sqrt{2\pi}})^{\otimes k} \right\rangle^2].
	\end{align*}
	Combining everything gives
	\begin{align*}
		\E\qty[\left\langle T, (\Pi^\star g(w))^{\otimes k} \right\rangle^2]
		&\ge \E\qty[\frac{1}{2}\left\langle T, \qty(\frac{H w}{\sqrt{2\pi}})^{\otimes k} \right\rangle^2] - \E[\delta(w)^2] \\
		&\ge \frac{1}{4}\E\qty[\left\langle T, \qty(\frac{H w}{\sqrt{2\pi}})^{\otimes k} \right\rangle^2] \\
		&\gtrsim d^{-k} \norm{\hat T}_F^2 \\
		&\ge d^{-k} \lambda_{min}(H)^{2k} \\
		&= (rd\kappa^2)^{-k}.
	\end{align*}
\end{proof}

\begin{corollary}\label{corollary:z_construction}
	Assume $n \ge Cd^2 r \kappa^2 \iota^{p+1}$ and $d \ge C \kappa r^{3/2}$ for a sufficiently large constant $C$. Then for any $k \le p$ and any symmetric $k$ tensor $T$ supported on $S^\star$, there exists $z_T(w)$ such that
	\begin{align*}
		\E_{w}[z_T(w)(g_n(w) \cdot x)^p] = \langle T,x^{\otimes k} \rangle
	\end{align*}
	and we have the bounds
	\begin{align*}
		\E_{w}[z_T(w)^2] \lesssim (rd\kappa^2)^k \norm{T}_F^2 \qand \abs{z_T(w)} \lesssim (rd\kappa^2)^k \norm{T}_F \norm{g_n(w)}^k.
	\end{align*}
\end{corollary}
\begin{proof}
	Let
	\begin{align*}
		z_T(w) := \vec(T)^T \mat(\E[g_n(w)^{\otimes 2k}])^\dagger \vec(g_n(w)^{\otimes k}).
	\end{align*}
	Note that $\vec(T) \in \spn(\mat(\E[g(w)^{\otimes 2k}]))$ by \Cref{lem:g_min_sval}. Therefore,
	\begin{align*}
		&\E_{w}[z_T(w)(g_n(w) \cdot x)^k \\
		&= \E_{w}\qty[\vec(T)^T \mat(\E[g_n(w)^{\otimes 2k}])^\dagger \vec(g(w)^{\otimes k}) \vec(g_n(w)^{\otimes k})^T \vec(x^{\otimes k})] \\
		&= \langle T, x^{\otimes k} \rangle.
	\end{align*}
	For the bounds on $z$ we have
	\begin{align*}
		&\E_{w} [z_T(w)^2] \\
		&= \E_{w} [\vec(T)^T \mat(\E[g_n(w)^{\otimes 2k}])^\dagger \vec(g(w)^{\otimes k})\vec(g_n(w)^{\otimes k})^T \mat(\E[g(w)^{\otimes 2k}])^\dagger \vec(T)] \\
		&= \Vec(T)^T \mat(\E[g_n(w)^{\otimes 2k}])^\dagger \vec(T) \\
		&\lesssim (rd\kappa^2)^k \norm{T}_F^2
	\end{align*}
	and
	\begin{align*}
		\abs{z_T(w)}
		&= \abs{\vec(T)^T \mat(\E[g_n(w)^{\otimes 2k}])^\dagger \vec(g_n(w)^{\otimes k})} \\
		&\lesssim (rd\kappa^2)^k \norm{T}_F \norm{g_n(w)}^k.
	\end{align*}
\end{proof}

\begin{lemma}\label{lem:perfect_features_single_tensor}
	Assume $n \ge Cd^2 r \kappa^2 \iota^{p+1}$ and $d \ge C \kappa r^{3/2}$ for a sufficiently large constant $C$. Let $\eta_1 = \sqrt{\frac{d}{C^2 \iota^3}}$, let $k \le p$ and let $T$ be a $k$ tensor. Then with high probability, there exists $h_T(a,w,b)$ such that if
	\begin{align*}
		f_{h_T}(x) := \E_{a,w,b}[h_T(a,w,b)\sigma(w^{(1)} \cdot x + b)]
	\end{align*}
	we have
	\begin{align*}
		\frac{1}{n} \sum_{i=1}^n (f_h(x_i) - \langle T, x_i^{\otimes p} \rangle)^2 &\lesssim \frac{1}{n}
	\end{align*}
	and the moment bounds
	\begin{align*}
		\E_{w,a,b}[h_T(a,w,b)^2] &\lesssim r^k \kappa^{2k} \iota^{3k} \norm{T}_F^2 \\
		\sup_w \abs{h_T(a,w,b)}] &\lesssim r^k \kappa^{2k} \iota^{6k} \norm{T}_F. 
	\end{align*}
\end{lemma}
\begin{proof}
	We define
	\begin{align*}
		h_T(a,w,b) := \frac{v_k(a,b) z_T(w)}{(2\eta)^k} \1_{\eta_1\norm{g_n(w)} \le 1} \prod_{i=1}^n \1_{\abs{g_n(w) \cdot x_i} \le 1}.
	\end{align*}
	where $v_k(a,b)$ and $z_T(w)$ are constructed in \Cref{corollary:v_construction} and \Cref{corollary:z_construction} respectively. Recall that $w^{(1)} = 2 \eta_1 a g_n(w)$. Then for $x \in \qty{x_1,\ldots,x_n}$,
	\begin{align*}
		&f_{h_T}(x) \\
		&= \frac{1}{(2\eta_1)^k}\E_{a,w,b}\qty[v_k(a,b) z_T(w) \sigma(2 \eta_1 a g_n(w) \cdot x + b)] \\
		&= \E_{w}\qty[z_T(w) \qty[(g_n(w) \cdot x)^k + O\qty(\abs{g_n(w) \cdot x}^k \qty[1 - \1_{\eta_1\norm{g_n(w)} \le 1}\prod_{i=1}^n \1_{\abs{g_n(w) \cdot x_i} \le 1}])]] \\
		&= \langle T,x^{\otimes k} \rangle + \poly(d)\qty[\P_{w}[\eta_1\norm{g_n(w)} \ge 1] + \sum_{i=1}^n \P_{w}[\abs{2 \eta_1 g_n(w) \cdot x_i} \ge 1]] \\
		&= \langle T,x^{\otimes k} \rangle + \poly(n,d)e^{-\iota} \\
		&= \langle T,x^{\otimes k} \rangle + O\qty(\frac{1}{n})
	\end{align*}
	where the second to last line followed from \Cref{lem:gn_xi_concentration}. The first part of the lemma now follows from a union bound over $x_1,\ldots,x_n$. For the bounds on $h$, we have
	\begin{align*}
		&\E_{a,w,b}[h(a,w,b)^2] \\
		&= \frac{1}{(2\eta_1)^{2k}}\E_{a,w,b}\qty[v_k(a,b)^2 z_T(w)^2] \\
		&\lesssim \eta_1^{-2k} (rd\kappa^2)^k \norm{T}_F^2 \\
		&= r^k \kappa^{2k} \iota^{3k} \norm{T}_F^2.
	\end{align*}
	and
	\begin{align*}
		\sup_w \abs{h(a,w,b)}
		&= \sup_w \abs{\frac{v_k(a,b) z_T(w)}{(2\eta)^k}\1_{\eta_1 \norm{g_n(w)} \le 1}} \\
		&\lesssim \eta_1^{-k} (rd\kappa^2)^k \norm{g_n(w)}^k \norm{T}_F \\
		&= \eta_1^{-2k} (rd\kappa^2)^k \norm{T}_F \\
		&\lesssim r^k \kappa^{2k} \iota^{6k} \norm{T}_F.
	\end{align*}
\end{proof}

\begin{corollary}\label{corollary:perfect_features}
	Assume $n \ge Cd^2 r \kappa^2 \iota^{p+1}$ and $d \ge C \kappa r^{3/2}$ for a sufficiently large constant $C$ and let $\eta_1 = \sqrt{\frac{d}{\iota^3}}$. Then with high probability, there exists $h(a,w,b)$ such that if
	\begin{align*}
		f_{h}(x) := \E_{a,w,b}[h(a,w,b)\sigma(w^{(1)} \cdot x + b)]
	\end{align*}
	we have
	\begin{align*}
		\frac{1}{n} \sum_{i=1}^n (f_h(x_i) - f^\star(x_i))^2 &\lesssim \frac{1}{n}
	\end{align*}
	and the moment bounds
	\begin{align*}
		\E_{w,a,b}[h_T(a,w,b)^2] &\lesssim r^p \kappa^{2p} \iota^{3p} \\
		\sup_w \abs{h_T(a,w,b)} &\lesssim r^p \kappa^{2p} \iota^{6p}.
	\end{align*}
\end{corollary}
\begin{proof}
	We know from \Cref{lem:f_taylor_decomposition} that
	\begin{align*}
		f^\star(x) = \sum_{k \le p} \langle T_k, x^{\otimes k} \rangle
	\end{align*}
	with $\norm{T_k}_F \lesssim r^\frac{p-k}{4}$. Let
	\begin{align*}
		h(a,w,b) := \sum_{k \le p} h_{T_k}(a,w,b).
	\end{align*}
	Then $\frac{1}{n} \sum_{i=1}^n (f_h(x_i) - f^\star(x_i))^2 \lesssim \frac{1}{n}$ is immediate from \Cref{lem:perfect_features_single_tensor} and
	\begin{align*}
		\E_{a,w,b}[h(a,w,b)^2] \lesssim \sum_{k \le p} \E_{a,w,b}[h_{T_k}(a,w,b)^2] \lesssim \sum_{k \le p} r^k \kappa^{2k} \iota^{3k} r^\frac{p-k}{2} \lesssim r^p \kappa^{2p} \iota^{3p}.
	\end{align*}
	and
	\begin{align*}
		\sup_{a,w,b} \abs{h(a,w,b)} \le \sum_{k \le p} \sup_{a,w,b}\abs{h_{T_k}(a,w,b)} \lesssim \sum_{k \le p} r^k \kappa^{2k} \iota^{6k} r^\frac{p-k}{2} \lesssim r^p \kappa^{2p} \iota^{6p}.
	\end{align*}
	complete the proof.
\end{proof}

\begin{lemma}\label{lem:perfect_features_finite_m}
	Assume $n \ge Cd^2 r \kappa^2 \iota^{p+1}$, $d \ge C \kappa r^{3/2}$, and $m \ge r^p \kappa^{2p} \iota^{6p+1}$ for a sufficiently large constant $C$ and let $\eta_1 = \sqrt{\frac{d}{\iota^3}}$. Then with high probability, there exists $a^\star \in \mathbb{R}^m$ such that if $\theta^\star = (a^\star,W^{(1)},b^{(1)})$,
	\begin{align*}
		\frac{1}{n} \sum_{i=1}^n (f_{\theta^\star}(x_i) - f^\star(x_i))^2 \lesssim \frac{1}{n} + \frac{r^p \kappa^{2p}\iota^{6p+1}}{m} \qand \norm{a^\star}^2 \lesssim \frac{r^p \kappa^{2p} \iota^{6p}}{m}.
	\end{align*}
\end{lemma}
\begin{proof}
	Let $a^\star_j := \frac{1}{m}h(a_j,w_j,b_j)$ where $h$ is the function constructed in \Cref{corollary:perfect_features}. Then,
	\begin{align*}
		\E_{i \in [n]}[(f_{\theta^\star}(x_i) - f^\star(x_i))^2]
		&\lesssim \E_{i \in [n]}[(f_{\theta^\star}(x_i) - f_h(x_i))^2] + \E_{i \in [n]}[(f_h(x_i) - f^\star(x_i))^2] \\
		&= \E_{i \in [n]}[(f_{\theta^\star}(x_i) - f_h(x_i))^2] + \frac{1}{n}.
	\end{align*}
	For $j \le m/2$, let
	\begin{align*}
		Z_{j}(x) := a_j^\star \sigma(w^{(1)}_j \cdot x + b_j) + a^\star_{m-j} \sigma(w^{(1)}_{m-j} \cdot x + b_{m-j}).
	\end{align*}
	Note that
	\begin{align*}
		f_{\theta^\star}(x) - f_h(x) = \sum_{j \le \frac{m}{2}} \qty(Z_j(x) - \E[Z_j(x)])
	\end{align*}
	and the $Z_j(x)$ are all i.i.d.. Let
	\begin{align*}
		\overline Z_j(x) := Z_j(x) \1_{\abs{w_j^{(1)} \cdot x} \le 1}	\1_{\abs{w_{m-j}^{(1)} \cdot x} \le 1}.
	\end{align*}
	Then with probability $1-\poly(n,m,d)e^{-\iota}$ we have that $Z_j(x_i) = \overline Z_j(x_i)$ for $i=1,\ldots,n$. Therefore,
	\begin{align*}
		f_{\theta^\star}(x) - f_h(x) = \sum_{j \le \frac{m}{2}} \qty(\overline Z_j(x) - \E \overline Z_j(x)) + \frac{m}{2} \qty[\E\overline Z_j(x) - \E Z_j(x)].
	\end{align*}
	For the first term, by Bernstein's inequality we have with probability at least $1-2e^{-\iota}$,
	\begin{align*}
		\sum_{j\le \frac{m}{2}} \overline Z_j(x) - \E[\overline Z_j(x)]
		&\lesssim \sqrt{\frac{\iota \E_{a,w,b}[h(a,w,b)^2]}{m}} + \frac{\iota r^p \kappa^{2p} \iota^{3p}}{m} \lesssim \sqrt{\frac{r^p \kappa^{2p} \iota^{6p+1}}{m}}.
	\end{align*}
	The second term is bounded as in the proof of \Cref{lem:perfect_features_single_tensor} by $\poly(n,d)e^{-\iota} \le \frac{1}{m}$ because $\P[w^{(1)} \cdot x > 1] \le e^{-\iota}$ from the choice of $\eta_1$. Therefore for any fixed $x$, with high probability we have
	\begin{align*}
		f_{\theta^\star}(x) = f^\star(x) + O\qty(\frac{1}{n} + \sqrt{\frac{r^p \kappa^{2p} \iota^{6p+1}}{m}})
	\end{align*}
	and the first part of the lemma follows from a union bound.
	
	We will now turn to the bound on $\norm{a^\star}^2$. Let $z_i = (a^\star_i)^2 + (a^\star_{m-i})^2$. Note that $\{z_i\}_{i \le m/2}$ are positive, i.i.d., and bounded by $O(m^{-2}r^{2p}\kappa^{4p}\iota^{12p})$. In addition, they have expectation $O(m^{-2} r^p\kappa^{2p}\iota^{3p})$. Therefore by Popoviciu's inequality they have variance bounded by
	\begin{align*}
		O\qty(\qty[m^{-1} r^p \kappa^{2p}\iota^{3p}]\qty[m^{-2} r^{2p} \kappa^{4p}\iota^{12p}]) = O\qty(m^{-3} r^{3p} \kappa^{6p} \iota^{15p}).
	\end{align*}
	Therefore by Bernstein's inequality we have that with high probability,
	\begin{align*}
		\norm{a^\star}^2
		&= \E[\norm{a^\star}^2] + O\qty(\frac{1}{m}\sqrt{\frac{r^{3p} \kappa^{6p} \iota^{15p}}{m}} + \frac{r^{2p}\kappa^{3p}\iota^{6p}}{m^2}) \lesssim \frac{r^p \kappa^{2p}\iota^{6p}}{m}.
	\end{align*}
\end{proof}

\subsection{Proof of \Cref{thm:sample_complexity}}\label{sec:thm1proof}
We will define
\begin{align*}
	\tilde{\mathcal{L}}(\theta)(\theta) := \frac{1}{n} \sum_{i=1}^n (f_\theta(x_i)-f^\star(x_i))^2.
\end{align*}
to be the empirical $L^2$ losses with respect to the true labels (recall $y_i = f^\star(x_i) + \epsilon_i$, $\epsilon_i \sim \{-\sigma,\sigma\}$).

\begin{lemma}\label{lem:thetastar_noisy_loss}
	Assume $n \ge Cd^2 r \kappa^2 \iota^{p+1}$ and $d \ge C \kappa r^{3/2}$ for a sufficiently large constant $C$ and let $\eta_1 = \sqrt{\frac{d}{\iota^3}}$. Let $a^\star$ be the vector constructed in the proof of \Cref{lem:perfect_features_finite_m} and let $\theta = (a^\star,W^{(1)},b^{(1)})$. Then with high probability,
	\begin{align*}
		\mathcal{L}(\theta) - \varsigma^2 \lesssim \frac{r^p \kappa^{2p} \iota^{6p+1}}{m} + \sqrt{\frac{\iota}{n}}.
	\end{align*}
\end{lemma}
\begin{proof}
	Let $\delta_i = f_\theta(x_i) - f^\star(x_i)$. Then,
	\begin{align*}
		\frac{1}{n}\norm{\delta + \epsilon}_2^2 = \frac{1}{n}\qty(\norm{\delta}^2 + 2\langle \delta,\epsilon \rangle + \norm{\epsilon}_2^2).
	\end{align*}
	First, by Hoeffding's inequality, we have with high probability,
	\begin{align*}
		\frac{\norm{\epsilon}^2}{n} \le \varsigma^2 + \frac{C\varsigma^2\sqrt{\iota}}{\sqrt{n}} = \varsigma^2 + O\qty(\sqrt{\frac{\iota}{n}}).
	\end{align*}
	Similarly, by Hoeffding's inequality we have with high probability, $\frac{1}{n}\langle \delta, \epsilon \rangle \le \varsigma \sqrt{\frac{2\iota\tilde{\mathcal{L}}(\theta)}{n}} = O\qty(\sqrt{\frac{\iota}{n}})$.
\end{proof}

\noindent We are now ready to directly prove \Cref{thm:sample_complexity}. 
\begin{proof}[Proof of \Cref{thm:sample_complexity}]
Note that we can assume that there is an absolute constant $C$ such that $n \ge Cd^2 r \kappa^2 \iota^{p+1}$, and $m \ge r^p \kappa^{2p} \iota^{6p+1}$. Otherwise, we can simply take $\lambda \to \infty$ and return the zero predictor.

From \Cref{lem:thetastar_noisy_loss} we know that with high probability, there exists $a^\star$ such that if $\theta = (a^\star,W^{(1)},b^{(1)})$,
\begin{align*}
	\mathcal{L}(\theta) - \zeta^2 \lesssim \frac{r^p \kappa^{2p} \iota^{6p+1}}{m} + \sqrt{\frac{\iota}{n}}.
\end{align*}
and $\norm{a^\star}_2^2 \lesssim \frac{r^{p}\kappa^{2p}\iota^{6p+1}}{m}.$ Therefore by equality of norm constrained linear regression and ridge regression, there exists $\lambda > 0$ such that if
\begin{align*}
	a^{(\infty)} = \min_a \mathcal{L}\qty((a,W^{(1)},b^{(1)})) + \lambda \frac{\norm{a}^2}{2},
\end{align*}
\begin{align*}
	\mathcal{L}\qty((a^{(\infty)},W^{(1)},b^{(1)})) \le \mathcal{L}\qty((a^\star,W^{(1)},b^{(1)})) \qand \norm{a_\infty} \le \norm{a^\star}.
\end{align*}
Note that we can approximate $a^{(\infty)}$ by $a^{(T)}$ to within arbitrary accuracy within $T = \tilde\Theta(\eta^{-1}\lambda^{-1})$ steps. Let
\begin{align*}
		\mathcal{F} = \qty{f_\theta ~:~ \norm{a}_2 \le \norm{a^\star}, \norm{w_j} \le 1}.
\end{align*}
Then with high probability, $f_{(a^{(T)},W^{(1)},b^{(1)})} \in \mathcal{F}$. In addition, from \Cref{lem:rademacher_2layernn},
\begin{align*}
	\sup_{f \in \mathcal{F}} \abs{\frac{1}{n}\sum_{i=1}^n \abs{f(x_i)-y_i} - \E_{x,y} \abs{f(x)-y}} 
	&\lesssim \sqrt{\frac{\norm{a^\star}^2 md}{n}} + \sqrt{\frac{\iota}{n}} \\
	&\lesssim \sqrt{\frac{d r^p\kappa^{2p}\iota^{6p}}{n}}.
\end{align*}
Therefore,
\begin{align*}
	&\E_{x,y}\abs{f_{\theta^{(T)}}(x) - y} - \E_{x,y}\abs{f^\star(x) - y} \\
	&\lesssim \sqrt{\frac{dr^p \kappa^{2p} \iota^{6p}}{n}} + \sqrt{\frac{r^p \kappa^{2p} \iota^{6p+1}}{m}} + \qty(\frac{\iota}{n})^{1/4}.
\end{align*}
which completes the proof.

\end{proof}

\section{Transfer Learning}
\begin{proof}[Proof of \Cref{thm:transferlearning}] The proof of \Cref{thm:transferlearning} is virtually identical to that of \Cref{thm:sample_complexity}. We can use \Cref{lem:perfect_features_finite_m} to construct $a^\star$ such that if $\theta^\star = (a^\star,W^{(1)},b^{(1)})$ then with high probability,
	\begin{align*}
		L(\theta^\star) - \varsigma^2 \lesssim \frac{r^p\kappa^{2p}\iota^{6p+1}}{m} + \sqrt{\frac{\iota}{N}} \qand \norm{a^\star}^2 \lesssim \frac{r^p\kappa^{2p}\iota^{6p}}{m}.
	\end{align*}
	In addition, there exists $\lambda$ such that if $T \ge \Theta(\eta^{-1}\lambda^{-1})$,
	\begin{align*}
		L(\theta^{(T)}) \le L(\theta^\star) \qand \norm*{a^{(T)}} \le \norm{a^\star}.
	\end{align*}
	Now let $\mathcal{F} = \qty{f_{(a,W,b)} ~:~ \norm{a}_2 \le \norm{a^\star}}$. Then by \Cref{lem:rademacher_linear} we have with high probability,
	\begin{align*}
		\E_{x,y}\abs{g_{a^{(T)}}(x) - y} - \varsigma
		&\le \tilde O\qty(\sqrt{\frac{r^p\kappa^{2p}}{N}} + \sqrt{\frac{r^p\kappa^{2p}}{m}} + \frac{1}{N^{1/4}}) \\
		&= \tilde O\qty(\sqrt{\frac{r^p\kappa^{2p}}{\min(m,N)}} + \frac{1}{N^{1/4}})
	\end{align*}
	which completes the proof.
\end{proof}

\section{Concentration Lemmas}
\begin{lemma}[Corollary of Lemma 1 in \citep{LaurentMassart2000}]\label{lem:chi_square}
	Let $X \sim \chi^2(d)$. Then, for any $t \ge 0$,
	\begin{align*}
		\P[X \ge d + 2\sqrt{dt} + 2t] &\le \exp(-t) \\
		\P[X \le d - 2\sqrt{dt}] &\le \exp(-t).
	\end{align*}
\end{lemma}

\begin{corollary}\label{cor:beta_upper_lower_bound}
	Let $w \sim N(0,I_d)$. Then for some constant $C$,
	\begin{align*}
		\P\qty[\frac{\norm{\Pi^\star w}^2}{\norm{w}^2} \in \qty[\frac{r}{Cd},\frac{Cr}{d}]] \gtrsim 1.
	\end{align*}
\end{corollary}

\begin{lemma}[Corollary 5.35 in \citet{vershynin2018high}]\label{lem:gaussian_operator_norm}
	Let $X \in \R^{n \times d}$ with $X_{ij} \sim N(0,1)$. Then with probability at least $1-2e^{-\iota}$,
	\begin{align*}
		\norm{X}_2 \le \sqrt{n} + \sqrt{d} + \sqrt{2\iota}.
	\end{align*}
\end{lemma}

\subsection{Polynomial Concentration}

\begin{lemma}\label{lem:poly_tail}
	Let $g$ be a polynomial of degree $p$. Then there exists an absolute constant $C_p$ depending only on $p$ such that for any $\delta$,
	\begin{align*}
		\P[\abs{g(x) - \E[g(x)]} \ge \delta\sqrt{\E[g(x)^2]}] \le 2\exp(-C_p \min(\delta^2,\delta^{2/p})).
	\end{align*}
\end{lemma}
\begin{proof}
	Note that by \Cref{lem:parseval},
	\begin{align*}
		\norm{\E[\nabla^k g(x)]}_{HS} \le \sqrt{k!}.
	\end{align*}
	Therefore by Theorem 1.2 of \citep{gotzePolynomial2021}, there exists an absolute constant $C_p$ such that
	\begin{align*}
		\P[\abs{g(x) - \E[g(x)]} \ge \delta \sqrt{\E[g(x)^2]}] \le 2\exp(-C_p\min_{1 \le s \le p} \delta^{2/s}) = 2\exp(-C_p \min(\delta^2,\delta^{2/p})).
	\end{align*}
\end{proof}

\begin{lemma}\label{lem:hyperplane_regions_net}
	Let $\sigma(x) \in \{x,\relu(x)\}$. There exists an absolute constant $C$ such that for any $x_1,\ldots,x_n \in \mathbb{R}^d$, there exists $\mathcal{N}^x_\epsilon,\pi$ with $\abs{\mathcal{N}^x_\epsilon} \le e^{ Cd \log(n/\epsilon)}$ such that for every $w \in S^{d-1}$, $\pi(w) \in \mathcal{N}^x_\epsilon$ ,$\sigma'(w \cdot x_i) = \sigma'(\pi(w) \cdot x_i)$ for $i=1,\ldots,n$ and $\norm{w - \pi(x)} \le \epsilon$.
\end{lemma}
\begin{proof}
	Note that the planes $w \cdot x_1 = 0, \ldots, w \cdot x_n = 0$ divides the sphere $S^{d-1}$ into at most $\sum_{i = 0}^d \binom{n}{i} \lesssim n^d$ convex regions. For each region there exists an $\epsilon$ net of size $\qty(\frac{3}{\epsilon})^d$. Therefore we can take the union of these nets over each region which has size at most $\qty(\frac{3n}{\epsilon})^d = e^{C d \log(n/\epsilon)}$.
\end{proof}

\begin{lemma}\label{lem:gradient_concentration}
	Let $f(x)$ be a polynomial of degree $p$ and let $\sigma(x) \in \{x,\relu(x)\}$. Then there exists an absolute constant $C_p$ depending only on $p$ such that for any $\iota > 0$, with probability at least $1-2ne^{-\iota}$, we have
	\begin{align*}
		\sup_{w \in S^{d-1}} \norm{\frac{1}{n} \sum_{i=1}^n f(x_i) x_i \sigma'(w \cdot x_i) - \E\qty[f(x)x\sigma'(w \cdot x)]} \le C_p \sqrt{\E[g(x)^2]} \sqrt{\frac{d\iota^{p+1}}{n}}.
	\end{align*}
\end{lemma}
\begin{proof}
	Note that we may assume $\iota \ge \log(2n)$ otherwise there is nothing to prove. Let $C$ be a sufficiently large absolute constant. We fix a truncation radius $R := (C \iota)^{p/2}$ and WLOG assume that $\E[f(x)^2] = 1$. Let
	\begin{align*}
		Y(w) :=\frac{1}{n}\sum_{i=1}^n f(x_i)x_i\sigma'(w \cdot x_i)
		\qand
		\widetilde{Y}(w) :=\frac{1}{n}\sum_{i=1}^n f(x_i) x_i \sigma'(w \cdot x_i) \1_{\{\abs{f(x_i)}\le R\}}.
	\end{align*}
	First, note that by \Cref{lem:poly_tail}, with probability at least $1-2e^{-2\iota}$, we have $\abs{f(x)} \le R$. Therefore by a union bound we have with probability at last $1-2ne^{-2\iota}$ we have $\abs{f(x_i)} \le R$ for $i=1,\ldots,n$. Conditioned on this event, $Y(w) = \tilde Y(w)$ uniformly over all $w \in S^{d-1}$.
	Next, we will bound $\sup_w \norm{\E_x[Y(w)] - \E_x[\widetilde Y(w)]}$:
	\begin{align*}
		\sup_w \norm{\E_x[Y(w)] - \E_x[\widetilde Y(w)]}
		&= \sup_w \norm{\E_x\qty[g(x) x \sigma'(w \cdot x) \1_{\{\abs{g(x_i)} > R\}}]} \\
		&\le \E_x\qty[\abs{f(x)} \norm{x}\1_{\{\abs{g(x_i)} > R\}}] \\
		&\le \E\qty[g(x)^2]^{1/2} \E\qty[\norm{x}^4]^{1/4} \P[\abs{g(x_i)} > R]^{1/4} \\
		&\le 2\sqrt{2d} \exp(\iota/2) \\
		&\lesssim \sqrt{\frac{d}{n}}.
	\end{align*}
	Finally, we concentrate $\sup_w \norm{\tilde Y(w) - \E_x[\tilde Y(w)]}$. Let $\epsilon = \sqrt{\frac{d}{n}}$, let $\mathcal{N}_{1/4}$ be a minimal $1/4$-net of $S^{d-1}$ with $|\mathcal{N}_{1/4}| \le e^{Cd}$ and let $\mathcal{N}^x_\epsilon$ be the net defined in \Cref{lem:hyperplane_regions_net} with $\abs{\mathcal{N}^x_\epsilon} \le e^{C d \log(n/\epsilon)}$ and let $\pi(w)$ be the projection function defined in \Cref{lem:hyperplane_regions_net}. Then because $\tilde Y(w) = \tilde Y(\pi(w))$,
	\begin{align*}
		&\sup_w \norm{\tilde Y(w) - \E_x[\tilde Y(w)]} \\
		&\le \sup_{w \in \mathcal{N}^x_\epsilon} \norm{\tilde Y(w) - \E_x[\tilde Y(w)]} + \sup_w \norm{\E_x[\tilde Y(w)] - \E_x[\tilde Y(\pi(w))]} \\
		&\le \sup_{w \in \mathcal{N}^x_\epsilon} \norm{\tilde Y(w) - \E_x[\tilde Y(w)]} + \sup_w \norm{\E_x[Y(w)] - \E_x[Y(\pi(w))]} + O\qty(\sqrt{\frac{d}{n}}).
	\end{align*}
	Next, because $w \to \E_x[Y(w)]$ is $O(1)$ Lipschitz (see \Cref{sec:appendix:expand_feature}), we can bound this by
	\begin{align*}
		\sup_w \norm{\tilde Y(w) - \E_x[\tilde Y(w)]} \le \sup_{w \in \mathcal{N}^x_\epsilon} \norm{\tilde Y(w) - \E_x[\tilde Y(w)]} + O\qty(\epsilon + \sqrt{\frac{d}{n}}).
	\end{align*}
	Therefore it remains to bound $\sup_{w \in \mathcal{N}^x_\epsilon} \norm{\tilde Y(w) - \E_x[\tilde Y(w)]}$. First, for fixed $w$ we have
	\begin{align*}
		\norm{\tilde Y(w) - \E_x[\tilde Y(w)]}
		= \sup_{u \in S^{d-1}} u \cdot \qty[\tilde Y(w) - \E_x[\tilde Y(w)]]
		\le 2\sup_{u \in \mathcal{N}_{1/4}} u \cdot \qty[\tilde Y(w) - \E_x[\tilde Y(w)]].
	\end{align*}
	Let $Z_i(w) := g(x_i) (u \cdot x_i)\sigma'(w \cdot x_i)\1_{\{\abs{g(x)} < R\}}$ so that
	\begin{align*}
		u \cdot \qty[\tilde Y(w) - \E_x[\tilde Y(w)]] = \frac{1}{n} \sum_{i=1}^n Z_i(w) - \E_x[Z_i(w)].
	\end{align*}
	Then note that for fixed $w$, $Z_i(w)$ is $R$-sub Gaussian so for each $u \in \mathcal{N}_{1/4}$, with probability $1-2e^{-z}$ we have
	\begin{align*}
		u \cdot \qty[\tilde Y(w) - \E_x[\tilde Y(w)]] \le R\sqrt{\frac{2z}{n}}.
	\end{align*}
	so by a union bound we have with probability $1-2e^{Cd\log(n/\epsilon)}e^{-z}$,
	\begin{align*}
		2\sup_{u \in \mathcal{N}_{1/4},w \in \mathcal{N}_\epsilon^x} u \cdot \qty[\tilde Y(w) - \E_x[\tilde Y(w)]] \le 2R\sqrt{\frac{2z}{n}}.
	\end{align*}
	so setting $z = C d \log(n/\epsilon) + \iota$ we have with probability $1-2e^{\iota}$,
	\begin{align*}
		2\sup_{u \in \mathcal{N}_{1/4},w \in \mathcal{N}_\epsilon^x} u \cdot \qty[\tilde Y(w) - \E_x[\tilde Y(w)]] \lesssim R\sqrt{\frac{d\log(n/\epsilon) + \iota}{n}}.	
	\end{align*}
	Using $\epsilon = \sqrt{\frac{d}{n}}$ and putting everything together gives with probability $1-2ne^{-\iota}$,
	\begin{align*}
		\sup_w \norm{Y(w) - \E[Y(w)]} \lesssim \sqrt{\frac{(d \log n + \iota)\iota^p}{n}} \lesssim \sqrt{\frac{d\iota^{p+1}}{n}}.
	\end{align*}
\end{proof}

\begin{lemma}\label{lem:gradient_noise_concentration}
	Let $\epsilon_i \sim \{-\varsigma,\varsigma\}$. Then with high probability,
	\begin{align*}
		\sup_w \norm{\frac{1}{n} \sum_{i=1}^n \epsilon_i x_i \sigma'(w \cdot x_i)} \lesssim \varsigma \sqrt{\frac{d\iota}{n}}.
	\end{align*}
\end{lemma}
\begin{proof}
	Note that
	\begin{align*}
		\sup_w \norm{\frac{1}{n} \sum_{i=1}^n \epsilon_i x_i \sigma'(w \cdot x_i)} = \sup_{u,w} \qty[\frac{1}{n} \sum_{i=1}^n \epsilon_i (u \cdot x_i) \sigma'(w \cdot x_i)].
	\end{align*}
	Next, note that for fixed $u,w$, $\epsilon_i (u \cdot x_i) \sigma'(w \cdot x_i)$ is $\varsigma^2$ sub-Gaussian so for any $\iota > 0$, with probability $1-2e^{-\iota}$,
	\begin{align*}
	\frac{1}{n} \sum_{i=1}^n \epsilon_i (u \cdot x_i) \sigma'(w \cdot x_i) \le \varsigma \sqrt{\frac{\iota}{n}}.
	\end{align*}
	Therefore,
	\begin{align*}
		\sup_{u,w} \qty[\frac{1}{n} \sum_{i=1}^n \epsilon_i (u \cdot x_i) \sigma'(w \cdot x_i)] \lesssim \sup_{u \in \mathcal{N}_{1/4}, w \in \mathcal{N}_{1/4}^x} \sum_{i=1}^n \epsilon_i (u \cdot x_i) \sigma'(w \cdot x_i)].
	\end{align*}
	By a union bound, with probability at least $1-2e^\iota$,
	\begin{align*}
	\sup_{u \in \mathcal{N}_{1/4}, w \in \mathcal{N}_{1/4}^x} \sum_{i=1}^n \epsilon_i (u \cdot x_i) \sigma'(w \cdot x_i)] \lesssim \varsigma \sqrt{\frac{d \log n + \iota}{n}} \lesssim \varsigma\sqrt{\frac{d\iota}{n}}
	\end{align*}
	which completes the proof.
\end{proof}

\begin{corollary}\label{cor:full_gradient_concentration}
	With high probability,
	\begin{align*}
		\sup_w \norm{g(w) - g_n(w)} \lesssim \sqrt{\frac{d\iota^{p+1}}{n}}.
	\end{align*}
\end{corollary}

\section{CSQ Lower Bound}\label{sec:appendix:CSQproof}
\begin{proof}[Proof of \Cref{lem:general_CSQ}]
	The proof is a modified version of the proof in \citet{Szrnyi2009CharacterizingSQ}. Let $\langle \cdot,\cdot \rangle_\mathcal{D}$ denote the $L^2$ inner product with respect to $\mathcal{D}$. We will show that there are at least two functions $f,g \in \mathcal{F}$ such that for each query $h_k$, $\abs{\langle f,h_k \rangle_\mathcal{D}} \le \tau$ and $\abs{\langle g,h_k \rangle_\mathcal{D}} \le \tau$. Therefore, we can simply respond to each query adversarially with $0$ and it is impossible for the learner to distinguish between $f,g$. Note that failing to do so will result in a loss of $\norm{f-g}_\mathcal{D}^2 \ge 2-2\epsilon$. Let the $k$th query be $h_k$ and let
	\begin{align*}
		A_k^+ = \qty{f \in \mathcal{F} ~:~ \langle f,h_k\rangle_\mathcal{D} \ge \tau} \qand A_k^- = \qty{f \in \mathcal{F} ~:~ \langle f,h_k\rangle_\mathcal{D} \le -\tau}
	\end{align*}
	Then by Cauchy-Schwarz we have
	\begin{align*}
		\abs{A_k^+}^2 \tau^2 \le \iprod{h_k,\sum_{f \in A_k^+} f}_\mathcal{D}^2 \le \norm{\sum_{f \in A_k^+} f}_\mathcal{D}^2 = \sum_{f,g \in A_k^+} \langle f,g \rangle_\mathcal{D} \le \abs{A_k^+} + \epsilon\qty(\abs{A_k^+}^2 - \abs{A_k^+})
	\end{align*}
	which implies
	\begin{align*}
		\abs{A_k^+} \le \frac{1-\epsilon}{\tau^2-\epsilon} \le \frac{1}{\tau^2 - \epsilon}.
	\end{align*}
	Similarly, we have that $\abs{A_k^-} \le \frac{1}{\tau^2 - \epsilon}$ so the number of functions that are eliminated from the $k$th query is at most $\frac{2}{\tau^2-\epsilon}$. We can continue this process for at most $\frac{\abs{F}(\tau^2-\epsilon)}{2}$ iterations.
\end{proof}

\begin{proof}[Proof of \Cref{lem:approx_orthogonal_vectors}]
	Let $v_1,\ldots,v_k \sim S^{d-1}$. Then for every pair $i \ne j$, $v_i \cdot v_j$ is $O(d^{-1})$ subgaussian so for an absolute constant $c$, with probability $1-2e^{-2c\epsilon^2 d}$, $\abs{v_i \cdot v_j} \le \epsilon$. Therefore with probability $1-k^2 e^{-2c\epsilon^2 d} > 0$ this holds for all $i \ne j$ so there must exist at least one collection of such points.
\end{proof}

\begin{proof}[Proof of \Cref{thm:CSQ_lower_bound}]
	Let $S$ be the set constructed in \Cref{lem:approx_orthogonal_vectors}. Let
	\begin{align*}
	\mathcal{F} = \qty{x \to \frac{He_p(v \cdot x)}{\sqrt{k!}} ~:~ v \in S}
	\end{align*}
	and note that for all $f \in \mathcal{F}$, $\norm{f}_\mathcal{D} = 1$.
	Then for $v,w \in S$ and $v \ne w$,
	\begin{align*}
		\abs{\iprod{\frac{He_k(v \cdot x)}{\sqrt{k!}},\frac{He_k(w \cdot x)}{\sqrt{k!}}}_\mathcal{D}} = \abs{(v \cdot w)^k} \le \epsilon^k.
	\end{align*}
	Therefore, by \Cref{lem:general_CSQ} we have for any $\epsilon$,
	\begin{align*}
		4q \ge e^{c\epsilon^2 d}(\tau^2 - \epsilon^k)
	\end{align*}
	In particular if we take $\epsilon = \sqrt{\frac{\log(4 q(cd)^{k/2})}{cd}}$ we get
	\begin{align*}
		\tau^2 \le \frac{1 + \log^{k/2} \qty(4q(cd)^{k/2})}{(cd)^{k/2}} \lesssim \frac{\log^{k/2} \qty(qd)}{d^{k/2}}.
	\end{align*}
\end{proof}

\section{Additional Technical Lemmas}

For a $k$ tensor $T$, let $\sym(T)$ denote the symmetrization of $T$ along all $k!$ permutations of indices.

\begin{lemma}\label{lem:f_taylor_decomposition}
	There exist $T_0,\ldots,T_p$ such that
	\begin{align*}
		f^\star(x) = \sum_{k \le p} \langle T_k, x^{\otimes k} \rangle
	\end{align*}
	and $\norm{T_k}_F \lesssim r^\frac{p-k}{4}$ for $k \le p$.
\end{lemma}
\begin{proof}
	Note that from the Taylor series of $f^\star(x)$ we have
	\begin{align*}
	T_k = \frac{\nabla^k f^\star(0)}{k!} = \sum_{j \le p-k} \frac{C_{j+k}(He_j(0))}{k!j!} = \sum_{2j \le p-k} \frac{(-1)^j (2j-1)!! C_{2j+k}(I^{\otimes j})}{k!(2j)!}.
	\end{align*}
	Therefore,
	\begin{align*}
	\norm{T_k}_F \lesssim \sum_{2j \le p-k} \norm{C_{2j+k}(I^{\otimes j})} \lesssim r^\frac{p-k}{4}.
	\end{align*}
\end{proof}

\subsection{Gaussian Lemmas}

\begin{lemma}\label{lem:gaussian_moment_tensor}
	\begin{align*}
		\E_{w \sim N(0,I_d)}[w^{\otimes 2k}] = (2k-1)!! \sym(I_d^{\otimes k})
	\end{align*}
\end{lemma}
\begin{proof}
	We will show equality for each coordinate. Let $i_1,\ldots,i_{2k}$ be an index set and let $c_1,\ldots,c_d$ be defined by $c_j = \abs{\qty{k ~:~ i_k = j}}$. First we will consider the case there is an odd $c_j$. Then, $\E_{w \sim N(0,I_d)}[w^{\otimes 2k}]_{i_1,\ldots,i_{2k}} = 0$ and $\qty[(2k-1)!! \sym(I_d^{\otimes k})]_{i_1,\ldots,i_{2k}} = 0$ because in order for this to be nonzero there must exist a pairing of $i_1,\ldots,i_{2k}$ such the numbers in each pair are identical.
	
	Next, assume that each $c_j$ is even. Then, $\E_{w \sim N(0,I_d)}[w^{\otimes 2k}]_{i_1,\ldots,i_{2k}} = \prod_{j=1}^d (c_j-1)!!$ by the standard formula for Gaussian moments. Finally, consider
	\begin{align*}
		\qty[(2k-1)!! \sym(I_d^{\otimes k})]_{i_1,\ldots,i_{2k}} = \frac{(2k-1)!!}{2k!}\sum_\sigma \1_{i_{\sigma_1}=i_{\sigma_2}} \cdots \1_{i_{\sigma_{2k-1}}=i_{\sigma_{2k}}}.
	\end{align*}
	Note that by a simple counting argument, the number of permutations such that this product of indicators is nonzero is exactly $k!\prod_{i=1}^d \frac{c_j!}{(c_j/2)!}$ as you can first order the indices corresponding to each $c_j$, then split them into groups of two, then shuffle these groups of two. Therefore,
	\begin{align*}
		\qty[(2k-1)!! I_d^k]_{i_1,\ldots,i_{2k}} = \frac{k!(2k-1)!!}{2k!}\prod_{i=1}^d \frac{c_j!}{(c_j/2)!} = \frac{1}{2^k}\prod_{i=1}^d (c_j-1)!! 2^{c_j/2} = \prod_{i=1}^d (c_j-1)!!
	\end{align*}
	because $\sum_j c_j = 2k$, which completes the proof.
\end{proof}

\begin{definition}
	Let $\{h_{kl}\}$ and $\{h^{-1}_{kl}\}$ denote the change of basis matrices between Hermite polynomials and monomials, i.e.
	\begin{align*}
		He_k(x) = \sum_{l \le k} h_{kl} x^l \qand x^k = \sum_{l \le k} h^{-1}_{kl} He_l(x).	
	\end{align*}
	Note that
	\begin{align*}
		h_{kl} =
		\begin{cases}
 			(-1)^\frac{k-l}{2} (k-l-1)!! \binom{k}{l} & 2 \mid k-l \\
 			0 & 2 \nmid k-l
		\end{cases}
		\qand 
		h^{-1}_{kl} =
		\begin{cases}
 			(k-l-1)!! \binom{k}{l} & 2 \mid k-l \\
 			0 & 2 \nmid k-l
		\end{cases}.
	\end{align*}
\end{definition}

\begin{lemma}\label{lem:gaussian_power_tensor_expansion} Let $T$ be a symmetric $p$-tensor and let $w \sim N(0,I_d)$. Then for $k \le p$,
	\begin{align*}
		\E \|T(w^{\otimes k})\|_F^2 = \sum_{2l \le k} (k-2l)! ((2l-1)!!)^2\binom{k}{2l}^2 \|T(I^{\otimes l})\|_F^2.
	\end{align*}
\end{lemma}
\begin{proof}
	Let $T = \sum_i c_i v_i^p$ with $\|v_i\| = 1$. Using the change of basis $x^k \to \sum_{l \le k} h^{-1}_{kl} He_l(x)$,
	\begin{align*}
	\E \|T(w^{\otimes k})\|_F^2
	&= \sum_{ij} c_i c_j \E[(w \cdot v_i)^k (w \cdot v_j)^k] (v_i \cdot v_j)^{p-k} \\
	&= \sum_{l \le k} l! (h^{-1}_{kl})^2 \sum_{ij} c_i c_i (v_i \cdot v_j)^{p-k+l} \\
	&= \sum_{2l \le k} (k-2l)! ((2l-1)!!)^2\binom{k}{2l}^2 \|T(I^{\otimes l})\|_F^2.
	\end{align*}
\end{proof}
\begin{corollary}\label{cor:tensor_power_r}
	Let $T$ be a symmetric $p$-tensor with $\dim(\spn(T)) = r$. For $k \le p$,
	\begin{align*}
		\E \|T(w^{\otimes k})\|_F^2 \lesssim r^{\lfloor\frac{k}{2}\rfloor} \norm{T}_F^2.
	\end{align*}
\end{corollary}
\begin{proof}
	The proof follows directly from \Cref{lem:gaussian_power_tensor_expansion} and the inequality $\|T(I^{\otimes l})\|_F = \|T(\Pi_{\spn(T)}^{\otimes l})\|_F \le \norm{T}_F^2 \norm{\Pi_{\spn(T)}^{\otimes l}}_F^2 = r^l \norm{T}_F^2$ for $2l \le k$.
\end{proof}

\begin{corollary}\label{lem:gaussian_power_frobenius_bound}
	Let $T$ be a symmetric $p$-tensor with $\dim(\spn(T)) = r$. With probability at least $1-2e^{-\iota}$,
	\begin{align*}
		\|T(w^{\otimes k})\|_F \lesssim \norm{T}_F \sqrt{r^{\lfloor\frac{k}{2}\rfloor} \iota^{k}}.
	\end{align*}
\end{corollary}
\begin{proof}
	Note that $F(w) = \norm{T(w^{\otimes k})}_F^2$ is a polynomial of degree $2k$. For $k \le p$, let $\tilde T_k$ be the $(k,k)$ tensor which comes from contracting the last $d-k$ indices of $T \otimes T$, i.e.
	\begin{align*}
		(\tilde T_k)_{i_1,\ldots,i_k}^{j_1,\ldots,j_k} = T_{i_1,\ldots,i_k,i_{k+1},\ldots,i_p}T^{j_1,\ldots,j_k,i_{k+1},\ldots,i_p}.
	\end{align*}
	Note that $F(w) = \E_w[\tilde T_k(w^{\otimes 2k})]$ and $\norm{\tilde T_k}_F \le \norm{T}_F^2$. Then by \Cref{lem:gaussian_power_tensor_expansion},
	\begin{align*}
		\E_w[F(w)^2] \lesssim \sum_{l \le k} \norm{\sym(\tilde T_k)(I^l)}_F^2 \lesssim \sum_{l \le \frac{k}{2}} \norm{T(I^l)}_F^4 \le \norm{T}_F^4 r^{2\lfloor \frac{k}{2} \rfloor}.
	\end{align*}
	Therefore by \Cref{lem:poly_tail}, with probability at least $1-2e^{-\iota}$, $F(w) \lesssim \norm{T}_F^2 r^{\lfloor \frac{k}{2} \rfloor} \iota^{k}$ and taking square roots completes the proof.
\end{proof}

\begin{corollary}
	For $k \le p$,
	\begin{align*}
		\E \|T(w^{\otimes k})\|_F^2 \le \E \langle T, w^{\otimes p} \rangle^2.
	\end{align*}
\end{corollary}
\begin{proof}
	This follows immediately from \Cref{lem:gaussian_power_tensor_expansion} and $(k-2l)!\binom{k}{2l}^2 \le (p-2l)!\binom{p}{2l}^2$.
\end{proof}

\begin{corollary}\label{cor:gauss_tensor_eval_lowerbound}
	Let $w \sim N(0,I_d)$. Then,
	\begin{align*}
		\E[w^{\otimes 2k}] \succeq k! \Pi_{\sym^k(\R^d)}
	\end{align*}
	where $\Pi_{\sym^k(\R^d)}$ denotes the projection onto symmetric $k$-tensors.
\end{corollary}
\begin{proof}
	Considering only the $l=0$ term in the above expansion of $\E[w^{\otimes 2k}](T,T)$ gives
	\begin{align*}
		\E[w^{\otimes 2k}](T,T) \ge k! \norm{T}_F^2.
	\end{align*}
\end{proof}

\begin{lemma}[Theorem 4.3 in \citep{pratowick2007}]\label{lem:gaussian_hypercontractivity}
	Let $f$ be a polynomial of degree $p$. Then
	\begin{align*}
		\E_{w \sim N(0,I_d)}[f(w)^k] \le O_{k,p}(1)\qty(\E_{w \sim N(0,I_d)}[f(w)^2])^{k/2}.
	\end{align*}
\end{lemma}

\subsection{Sphere Lemmas}

\begin{lemma}\label{lem:chi_moment}
	Let $\nu \sim \chi(d)$. Then,
	\begin{align*}
		\E[\nu^{2k}] = \prod_{j=0}^{k-1} (d+2j) = d(d+2)\cdots(d+2k-2) = \Theta(d^k).
	\end{align*}
\end{lemma}

\begin{lemma}
	Let $\bw \sim S^{d-1}$. Then,
	\begin{align*}
		\E\qty[\bw^{\otimes 2k}] = \frac{\E_{w \sim N(0,I_d)}[w^{2k}]}{\E_{\nu \sim \chi(d)}[\nu^{2k}]}.
	\end{align*}
\end{lemma}
\begin{proof}
	This follows from the decomposition $w = \nu \bw$ with $\nu \sim \chi(d),\bw \sim S^{d-1}$ independent.
\end{proof}

\begin{corollary}\label{lem:sphere_power_frobenius_bound}
	Let $T$ be a symmetric $p$-tensor with $\dim(\spn(T)) = r$. With probability at least $1-2e^{-\iota}$,
	\begin{align*}
		\|T(\bw^{\otimes k})\|_F \lesssim \norm{T}_F \sqrt{\frac{r^{\lfloor\frac{k}{2}\rfloor} \iota^{k}}{d^k}}.
	\end{align*}
\end{corollary}

\begin{corollary}\label{lem:sphere_power_C2_growth}
	Let $\bw \sim S^{d-1}$. For $k \le p$,
	\begin{align*}
		\E \|T(\bw^{\otimes k})\|_F^2 \lesssim d^{p-k}\E \langle T, \bw^{\otimes p} \rangle^2.
	\end{align*}
\end{corollary}

\subsection{Rademacher Complexity Bounds}

\begin{lemma}\label{lem:rademacher_linear}
	Let $f = a^T \sigma(W^\star x + b)$ be a two layer neural network. For fixed $W,b$, Let
	\begin{align*}
		\mathcal{F} = \qty{f_{(a,W,b)} ~:~ \norm{a}_2 \le B_a}.
	\end{align*}
	Then,
	\begin{align*}
		\mathfrak{R}_n(\mathcal{F}) \le \sqrt{\frac{B_a^2(\|W\|_F^2+\|b\|^2)}{n}}.
	\end{align*}
\end{lemma}

\begin{proof}
	\begin{align*}
		\mathfrak{R}_n(\mathcal{F})
		&= \E_{x,\sigma}\qty[\sup_{f \in \mathcal{F}} \qty[\frac{1}{n} \sum_i \sigma_i \qty(a^T \sigma(W x_i + b))]] \\
		&= \frac{B_a}{n} \E_{x,\sigma}\qty[\norm{\sum_i \sigma_i \sigma(W x_i + b)}_2] \\
		&\le \frac{B_a}{n} \sqrt{\E_{x,\sigma}\qty[\norm{\sum_i \sigma_i \sigma(W x_i + b)}_2^2]} \\
		&= \frac{B_a}{\sqrt{n}} \sqrt{\E_{x} \norm{\sigma(W x_1 + b)}_2^2} \\
		&\le \sqrt{\frac{B_a^2(\|W\|_F^2+\|b\|^2)}{n}}.
		\end{align*}
\end{proof}

\begin{lemma}\label{lem:rademacher_2layernn}
	Let $\theta = (a,W,b)$ and let $f = a^T \sigma(W x + b)$ be a two layer neural network. Let
	\begin{align*}
		\mathcal{F} = \qty{f_\theta ~:~ \norm{a}_2 \le B_a, \norm{w_j} \le B_w}.
	\end{align*}
	Then,
	\begin{align*}
		\mathfrak{R}_n(\mathcal{F}) \le 2 B_a B_w\sqrt{\frac{md}{n}}.
	\end{align*}
\end{lemma}

\begin{proof}
	\begin{align*}
		\mathfrak{R}_n(\mathcal{F})
		&= \E_{x,\sigma}\qty[\sup_{f \in \mathcal{F}} \qty[\frac{1}{n} \sum_i \sigma_i \qty(a^T \sigma(Wx_i + b))]] \\
		&= \frac{B_a}{n} \E_{x,\sigma}\qty[\sup_{f \in \mathcal{F}} \norm{\sum_i \sigma_i \sigma(Wx_i + b)}_2] \\
		&\le \frac{B_a \sqrt{m}}{n} \E_{x,\sigma}\qty[\sup_{f \in \mathcal{F}} \norm{\sum_i \sigma_i \sigma(Wx_i + b)}_\infty] \\
		&= \frac{B_a \sqrt{m}}{n} \E_{x,\sigma}\qty[\sup_{f \in \mathcal{F}} \abs{\sum_i \sigma_i \sigma(w_j \cdot x_i + b_j)}] \\
		&\le \frac{2 B_a \sqrt{m}}{n} \E_{x,\sigma}\qty[\sup_{f \in \mathcal{F}} \sum_i \sigma_i \sigma(w_j \cdot x_i + b_j)] \\
		&\le \frac{2 B_a \sqrt{m}}{n} \E_{x,\sigma}\qty[\sup_{f \in \mathcal{F}} \sum_i \sigma_i (w_j \cdot x_i)] \\
		&\le 2 B_a B_w\sqrt{\frac{md}{n}}.
		\end{align*}
\end{proof}